\documentclass[11pt]{article}
\usepackage[margin=1in]{geometry}
\usepackage{times} 

\pdfoutput=1
\usepackage[T1]{fontenc}
\usepackage{lmodern}
\usepackage[protrusion=true,expansion=true]{microtype}
\usepackage{amsmath,amssymb,amsfonts,amsthm}
\usepackage{dsfont}
\usepackage{mathtools}
\usepackage{adjustbox}
\usepackage{thm-restate}
\usepackage{nicefrac}
\usepackage{caption}
\usepackage{subcaption}
\usepackage{graphicx}
\usepackage{fullpage}
\usepackage{booktabs}
\usepackage{xurl}
\usepackage{hyperref}
\usepackage[capitalise,noabbrev,nameinlink]{cleveref}
\usepackage{color}
\usepackage{wrapfig}
\usepackage{tikz}
\usetikzlibrary{decorations.pathreplacing}
\usepackage{setspace}
\usepackage{algorithm}
\usepackage[noend]{algpseudocode}
\usepackage[inline,shortlabels]{enumitem}
\usepackage{makecell}
\usepackage{comment}
\usepackage[framemethod=tikz]{mdframed}
\usepackage{xspace}
\usepackage{pgfplots}
\usepackage{framed}
\usepackage{thmtools}
\usepackage{thm-restate}
\pgfplotsset{compat=1.5}


\crefname{ALC@line}{line}{lines} %
\Crefname{ALC@line}{Line}{Lines} %

\newtheorem{theorem}{Theorem}[section]

\newtheorem{lemma}[theorem]{Lemma}

\newtheorem{definition}[theorem]{Definition}

\newenvironment{proofof}[1]{\begin{trivlist} \item {\bf Proof
#1:~~}}
  {\qed\end{trivlist}}

\newcommand{\namedref}[2]{\hyperref[#2]{#1~\ref*{#2}}}
\newcommand{\thmlab}[1]{\label{thm:#1}}
\newcommand{\thmref}[1]{\namedref{Theorem}{thm:#1}}

\newcommand{\seclab}[1]{\label{sec:#1}}

\newcommand{\deflab}[1]{\label{def:#1}}
\newcommand{\defref}[1]{\namedref{Definition}{def:#1}}

\def \OPT    {\mdef{\mathsf{OPT}}}

\DeclarePairedDelimiter{\card}{\lvert}{\rvert}
\DeclarePairedDelimiter{\iprod}{\langle}{\rangle}

\newcommand{\Encode}{\texttt{Encode}\xspace}
\newcommand{\Decode}{\texttt{Decode}\xspace}
\newcommand{\DPCluster}{\texttt{DP-Cluster}\xspace}
\newcommand{\DPMean}{\texttt{DP-Mean}\xspace}
\newcommand{\DPCovariance}{\texttt{DP-Covariance}\xspace}
\newcommand{\DPFilterEmbedding}{\texttt{DP-FilterEmbedding}\xspace}
\newcommand{\DPFilterImage}{\texttt{DP-FilterImage}\xspace}
\newcommand{\Embedding}{\texttt{Embedding}\xspace}
\newcommand{\Image}{\texttt{Image}\xspace}

\newcommand\norm[1]{\left\lVert#1\right\rVert}
\newcommand{\PPr}[1]{\ensuremath{\mathbf{Pr}\left[#1\right]}}

\newcommand{\Ex}[1]{\ensuremath{\mathbb{E}\left[#1\right]}}
\newcommand{\EEx}[2]{\ensuremath{\underset{#1}{\mathbb{E}}\left[#2\right]}}

\newcommand{\eps}{\varepsilon}

\def \calA    {\mdef{\mathcal{A}}}

\def \calD    {\mdef{\mathcal{D}}}
\def \calE    {\mdef{\mathcal{E}}}

\def \calM    {\mdef{\mathcal{M}}}
\def \calN    {\mdef{\mathcal{N}}}

\def \calR    {\mdef{\mathcal{R}}}

\def \calX    {\mdef{\mathcal{X}}}
\def \calY    {\mdef{\mathcal{Y}}}
\def \calZ    {\mdef{\mathcal{Z}}}

\def \bc    {\mdef{\mathbf{c}}}

\def \bw    {\mdef{\mathbf{w}}}

\def \bx    {\mdef{\mathbf{x}}}
\def \by    {\mdef{\mathbf{y}}}
\def \bz    {\mdef{\mathbf{z}}}

\newcommand{\bbR}{\mathbb{R}}
\newcommand{\ones}{\mathds{1}}

\newcommand{\mdef}[1]{{\ensuremath{#1}}\xspace}  %
\DeclareMathOperator*{\argmin}{argmin}

\DeclareMathOperator*{\poly}{poly}

\DeclareMathOperator*{\Lap}{Lap}
\DeclareMathOperator*{\GMM}{GMM}
\DeclareMathOperator*{\Tr}{Tr}
\DeclareMathOperator*{\TV}{TV}

\newcommand{\abs}[1]{\mdef{\left|#1\right|}}         %
\newcommand{\set}[1]{\mdef{\left\{#1\right\}}}                        %

\newcommand{\ignore}[1]{}

\newif\ifnotes\notestrue %
\ifnotes
\newcommand{\samson}[1]{\textcolor{purple}{{\bf (Samson:} {#1}{\bf ) }} \marginpar{\tiny\bf
             \begin{minipage}[t]{0.5in}
               \raggedright S:
            \end{minipage}}}            							
\else
\newcommand{\samson}[1]{}
\fi

\makeatletter
\renewcommand*{\@fnsymbol}[1]{\textcolor{mahogany}{\ensuremath{\ifcase#1\or *\or \dagger\or \ddagger\or
 \mathsection\or \triangledown\or \mathparagraph\or \|\or **\or \dagger\dagger
   \or \ddagger\ddagger \else\@ctrerr\fi}}}
\makeatother

\providecommand{\email}[1]{\href{mailto:#1}{\nolinkurl{#1}\xspace}}

\definecolor{mahogany}{rgb}{0.75, 0.25, 0.0}
\definecolor{darkblue}{rgb}{0.0, 0.0, 0.55}
\definecolor{darkpastelgreen}{rgb}{0.01, 0.75, 0.24}
\definecolor{bleudefrance}{rgb}{0.19, 0.55, 0.91}
\definecolor{darkgreen}{rgb}{0.0, 0.2, 0.13}
\definecolor{darkgoldenrod}{rgb}{0.72, 0.53, 0.04}
\definecolor{darkred}{rgb}{0.55, 0.0, 0.0}
\definecolor{mydarkblue}{rgb}{0,0.08,0.45}
\hypersetup{ %
    colorlinks=true,
    linkcolor=mydarkblue,
    citecolor=mydarkblue,
    filecolor=mydarkblue,
    urlcolor=mydarkblue,
}

\usepackage[natbib=true, backref=true, backend=biber, style=alphabetic, sortcites=true, sorting=alphabeticlabel,
minbibnames=3, maxbibnames=999, mincitenames=3, maxcitenames=3, minalphanames=3, maxalphanames=3]{biblatex}
\DeclareSortingScheme{alphabeticlabel}{
  \sort[final]{\field{labelalpha}} %
  \sort{
    \field{year}   %
  }
  \sort{
    \field{title}  %
  }
}
\AtBeginRefsection{\GenRefcontextData{sorting=ynt}}
\AtEveryCite{\localrefcontext[sorting=ynt]}

\title{Private Training \& Data Generation\\by Clustering Embeddings}
\author{%
  Felix Zhou\\
  Yale University\\
  \texttt{felix.zhou@yale.edu} \\
  \and
  Samson Zhou\\
  Texas A\&M University\\
  \texttt{samsonzhou@gmail.com} \\
  \and
  Vahab Mirrokni\\
  Google Research\\
  \texttt{mirrokni@google.com} \\
  \and
  Alessandro Epasto\\
  Google Research\\
  \texttt{aepasto@google.com} \\
  \and
  Vincent Cohen-Addad\\
  Google Research\\
  \texttt{cohenaddad@google.com} \\
}

\date{}

\begin{document}

\maketitle

\begin{abstract}
Deep neural networks often use large, high-quality datasets to achieve high performance on many machine learning tasks. When training involves potentially sensitive data, this process can raise privacy concerns, as large models have been shown to unintentionally memorize and reveal sensitive information, including reconstructing entire training samples. Differential privacy (DP) provides a robust framework for protecting individual data and in particular, a new approach to privately training deep neural networks is to approximate the input dataset with a privately generated synthetic dataset, before any subsequent training algorithm. We introduce a novel principled method for DP synthetic image embedding generation, based on fitting a Gaussian Mixture Model (GMM) in an appropriate embedding space using DP clustering. Our method provably learns a GMM under separation conditions. Empirically, a simple two-layer neural network trained on synthetically generated embeddings achieves state-of-the-art (SOTA) classification accuracy on standard benchmark datasets. 
Additionally, we demonstrate that our method can generate realistic synthetic images that achieve downstream classification accuracy comparable to SOTA methods. Our method is quite general, as the encoder and decoder modules can be freely substituted to suit different tasks. It is also highly scalable, consisting only of subroutines that scale linearly with the number of samples and/or can be implemented efficiently in distributed systems. 
\end{abstract}

\begin{figure}[htbp]
    \centering
    \includegraphics[width=.86\linewidth]{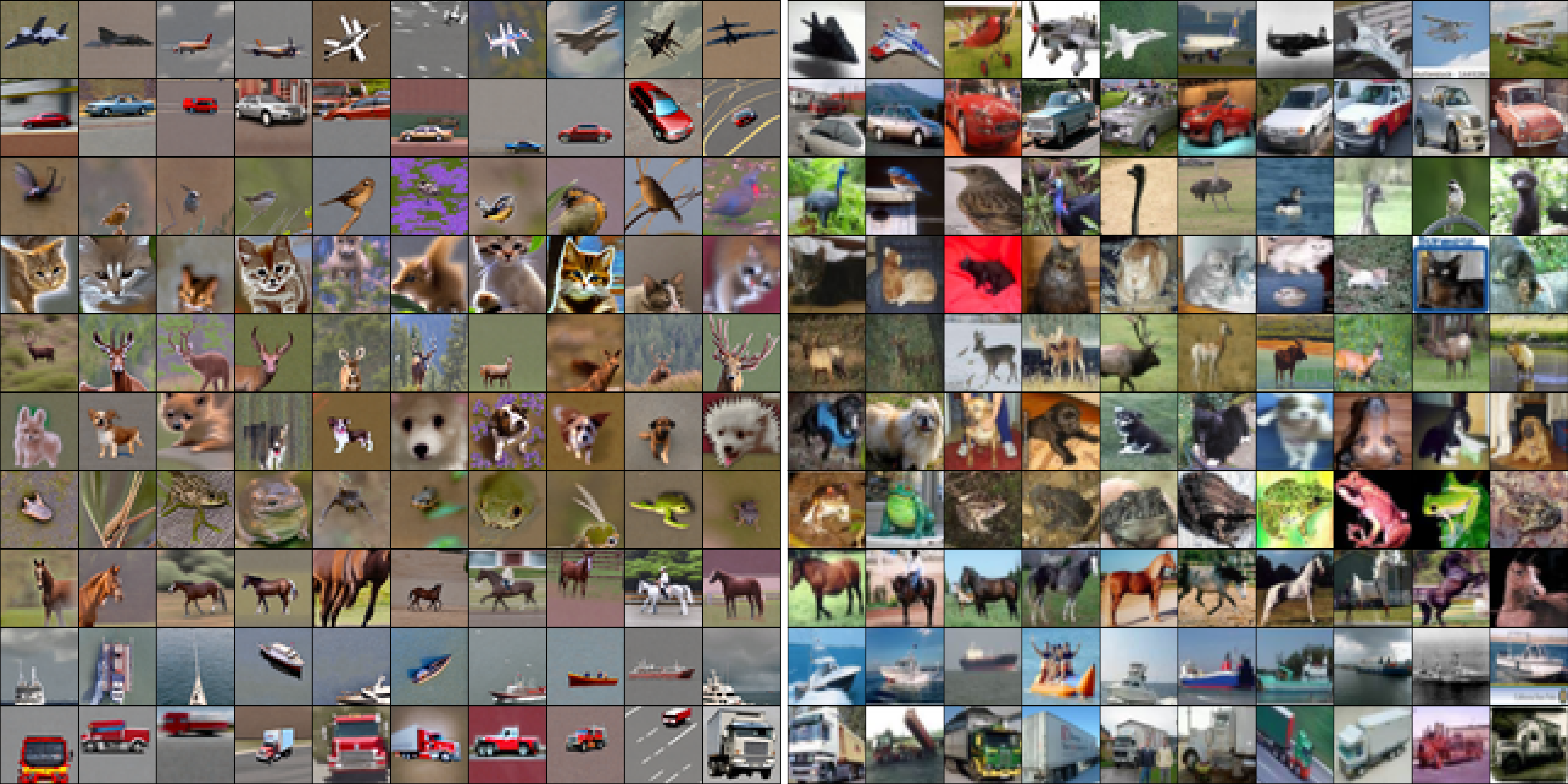}
    \caption{Synthetic and original CIFAR-10 images at \mbox{$\eps=8$}, \mbox{$\delta=10^{-5}$}.
    Each row corresponds to a different class.
    The left-most columns are synthetic images obtained with our method, while the right-most columns are original images.}
    \label{fig:cifar10-eps8}
\end{figure}

\section{Introduction}
\allowdisplaybreaks

The rise of massive datasets and increasingly complex machine learning (ML) models has transformed a large number of fields such as computer vision, natural language processing, and pattern recognition. 
These advancements have been fueled by the availability of high-quality datasets, enabling deep neural networks to achieve unprecedented performance across diverse tasks. 
However, the widespread reliance on large-scale data in ML introduces significant challenges and potential risks. 
One such risk is inadvertently exposing private user information in the output of a machine learning system~\cite{song2021systematic}. 
These risks have led to the establishment of strict data privacy regulations that forbid the storage of data that can be re-traced to individuals (re-identification)~\cite{voigt2017eu}.
Thus, privacy-preserving machine learning is no longer only a desirable property, but a necessity. 
When dealing with private data, differential privacy (DP)~\citep{DworkMNS06} (\defref{def:dp}) has emerged as the gold standard for ensuring strong privacy protection.
DP ensures that outputs of an algorithm are statistically similar regardless of the inclusion of any single data point, 
thus provably avoiding privacy risks such as re-identification.
As such, DP presents a strong framework for regulation-compliant training on sensitive data~\cite{cummings2018role}.

In this paper, we study the problem of differentially-private synthetic data generation~\citep{LinGKNY24,xie2024text,hou2024PrEText,ghalebikesabi2023diffusion,harder2023pretrained,yue2023simple,kurakin2023harnessing,torfi2022medical,mattern2022secure,rosenblatt2020differentially,torkzadehmahani2019DPCGAN,abay2018privacy,amin2025clustering,TanXXHW25}.
The goal of DP synthetic data release is to privately obtain approximations of potentially sensitive datasets that effectively extract, from the data, the useful information needed to achieve the system's goals, while at the same time ensuring that no individual's privacy is compromised.  

Specifically, consider the problem of training a machine learning model for a certain task (e.g., classification) with DP guarantees~\citep{AbadiCGMMT016}. In this context, DP synthetic data can be used to output a privatized version of the training dataset (see \Cref{fig:cifar10-eps8}), where then arbitrary \emph{non}-private training techniques can be applied without additional privacy risk.  
This approach is an increasingly popular alternative to directly training an ML model for the task using DP-SGD~\citep{malladi2023finetuning,he2023limits,YuNBGI0KLMWYZ22,de2022unlocking,AbadiCGMMT016} due to several advantages over private model training:
\begin{enumerate}%
\item 
Publishing synthetic datasets can enable direct inspection of an approximation of the underlying data,
allowing model designers the freedom to explore the data to identify issues, debug model behaviors, and assess data quality. 
\item 
DP synthetic data generation allows plug-in use of any existing model architecture without the need to run more complex privacy-preserving training methods, such as Differentially Private Stochastic Gradient Descent (DP-SGD)~\citep{AbadiCGMMT016,de2022unlocking}. 
This avoids the additional engineering effort needed to support DP training pipelines, which may require a fine-grained understanding of the interplay between privacy and the underlying training mechanics. 
For instance, Opacus \citep{yousefpour2021opacus}, a popular DP training library, requires a custom implementation of a per-sample gradient calculator for custom layers.
\item 
Private synthetic data release allows for unlimited training of models without incurring additional privacy costs. 
By comparison, repeated private training requires accounting for accumulated privacy loss.
\end{enumerate}

In this work,  we design novel methods for practical DP synthetic data generation by taking inspiration from the embedding clustering literature~\cite{xie2016unsupervised}.
Our guiding insight is that an appropriate embedding of the input data makes it more amenable to clustering.
Indeed, embedding objects into an appropriate space
and clustering these embeddings has been 
theoretically~\cite{spielman2025spectral,luxburg2007tutorial}
and empirically~\cite{huang2014deep,jiang2017variational,rozemberczki2019gemsec,reimers2019classification} shown to be effective at capturing desirable structures within data.

\subsection{Problem Definition}
Our overarching goal is to develop private synthetic data to perform downstream tasks. 
Though our methodologies focus on image classification, we first provide a formal model to quantify the performance of a synthetic dataset for general classification. 
Formally, consider an ML task where the goal is to perform classification on an input space $\calX$ for a label space $\calY=\{1,\ldots,L\}$.  
We remark that $\calX$ can either be the original input dataset or an embedding of the dataset under any fixed encoding scheme. 

Suppose the loss function $\ell(\cdot,\cdot;\theta):\calX\times\calY\to\mathbb{R}$ is parameterized by the vector $\theta$ over the hypothesis class. 
Here, the vector $\theta$ denotes the parameters of the deep learning model, e.g., the weights, biases, and hyperparameters of the model. 
We assume access to a collection $S$ of $n$ data points that are sampled i.i.d. over the space $\calZ=\calX\times\calY$ so that $(x_i,y_i)\sim \calD_\calZ$ for each $i\in[n]$, for some probability distribution $\calD_\calZ$. 
We would like to apply a learning algorithm $\calA$ onto the input $S$ to learn a model that can accurately predict the correct labels for new, unseen data by capturing the underlying patterns or relationships in the training data, while simultaneously protecting potentially sensitive information. 
Our approach is to privately estimate the true distribution,
say with $\widetilde{\calD}_\calZ$,
and release a private synthetic dataset $\widetilde{S}$ from $\widetilde{\calD}_\calZ$,
such that any algorithm $\calA$ trained on $\widetilde{S}$ retains good accuracy when applied to $\calD$ rather than $\widetilde{S}$. 
Quantitatively, our goal is to minimize the classification loss of the private synthetic dataset, which can be decomposed as follows:
\begin{align}
\EEx{\bx,y\sim \calD_{\calZ}}{\ell(\bx,y; \theta)}
&\le \underbrace{\frac{1}{|\widetilde{S}|} \sum_{\tilde\bx, \tilde y\in\widetilde{S}} \ell(\tilde \bx,\widetilde{y}; \theta)}_{\text{training error}}
+ \underbrace{\left| \frac{1}{|\widetilde{S}|} \sum_{\tilde\bx, \tilde y\in\widetilde{S}} \ell(\tilde \bx,\widetilde{y}; \theta) - \EEx{\tilde\bx, \tilde y\sim \widetilde{\calD}_{\calZ}}{\ell(\tilde\bx, \tilde y; \theta)} \right|}_{\text{synthetic data generation error}} \nonumber\\
&\qquad+ \underbrace{\left| \EEx{\tilde\bx, \tilde y\sim \widetilde{\calD}_{\calZ}}{\ell(\tilde\bx, \tilde y; \theta)} - \EEx{\bx,  y\sim \calD_{\calZ}}{\ell(\bx,  y; \theta)} \right|}_{\text{estimation error}}\,, \label{eq:classification-loss-decomposition}
\end{align}
where $\widetilde{S}=\{\widetilde{\bx_j},\widetilde{y_j}\}$ is a (private) synthetic dataset. 

Thanks to the postprocessing property of DP,
we can sample as many points from our privately estimated distribution $\widetilde{\calD}_\calZ$ as desired.
Hence we do not focus on the generation error
but note that under mild assumptions,
we can show that the generation error converges uniformly to 0 across all parameters $\theta\in \Theta$
using techniques such as metric entropy~\cite{wainwright2019high}.
Thus our primary concern is to develop a principled DP distribution estimation algorithm.

\paragraph{Loss function regularity conditions.}
We remark that although the problem formulation is simple, there is no upper bound to the private synthetic data loss without additional assumptions on the loss function. 
To circumvent these limitations due to poorly behaved loss functions, 
existing works often assume the dataset lies in a metric space that is ``well-behaved'' with respect to the loss function of the model. 
For example, \citet{SenerS18} assumes the loss function is $\lambda$-Lipschitz, i.e., \mbox{$|\ell(\bx, y)-\ell(\bx', y)|\le\lambda\cdot\|\bx-\bx'\|_2$}, while \citet{AxiotisCHJMSWW24} assumes the loss function is $(z,\lambda)$-H\"{o}lder continuous, i.e., \mbox{$|\ell(\bx, y)-\ell(\bx', y)|\le\lambda\cdot\|\bx-\bx'\|_2^z$} and demonstrate experimentally that this holds true for
in the context of large language models (T5-model~\cite{raffel2020exploring} for a translation task and BERT embeddings~\cite{devlin2019bert}).

\paragraph{Embedding space.} 
Functionally, the elements of the dataset can be embedded into a metric space, 
e.g., graph embeddings~\citep{GroverL16}, 
word embeddings~\citep{DevlinCLT19,PenningtonSM14,MikolovCCD13}, 
or image embeddings~\citep{RadfordKHRGASAM21,he2016resnet,SimonyanZ14a}. 
In general, an embedding can be acquired from the last layers of a neural network, which is especially appropriate when the model has already been pre-trained on publicly available data and the goal is to either fine-tune the model on private data for a specific task.
In these settings, a natural view is that the input dataset to the algorithm is the embedding of the original dataset, while the loss function may be the norm of the gradient of the embedding.

\subsection{Our Contributions}
In this paper, we present a novel training-free approach based on Gaussian mixture models (GMMs, c.f., \defref{def:GMM}) to privately generate synthetic data to minimize the error specified in \Cref{eq:classification-loss-decomposition} after training. 
We first seek to privately partition the input dataset into $k$ clusters, adapting a recent line of work~\citep{SenerS18,AxiotisCHJMSWW24} for the non-private active learning problem. 
Existing works use the resulting clustering to sample a number of points from each cluster, a procedure that inherently violates differential privacy. 
Instead, we privately estimate the intra-cluster covariance as our goal is to release a private synthetic dataset based on the resulting clustering. 
Informally, we would like to preserve the distribution of the input points, since the sample distribution serves as an estimate of the true distribution. 
The main intuition is that if the dataset can be partitioned into $k$ clusters such that {each cluster can be well-approximated by a Gaussian distribution}, then by generating data points using a GMM, we expect that the distributional distance between generated points and the input distribution to be small.
\begin{theorem}[Informal Parameter Estimation; See \Cref{thm:dp-GMM-estimation}]\label{thm:informal:GMM-estimation}
    Let $\eps, \delta, \alpha, \beta\in(0,1)$. 
    Given $n$ samples from a well-separated $k$-Gaussian mixture model $\calD_{\GMM}$ in $d$-dimensional space
    for $n=\poly(k, d, \nicefrac1{\alpha}, \nicefrac1\eps, \log(\nicefrac1\beta), \log(\nicefrac1\delta))$, 
    \Cref{alg:dp-synthetic-generation} is an $(\eps,\delta)$-DP algorithm that outputs parameter estimates $\hat w_i, \hat\mu_i, \hat\Sigma_i$ such that with probability $1-\beta$,
    $\norm{w-\hat w}_1, \norm{\mu - \hat\mu}_2, \norm{\Sigma-\hat\Sigma}_F\leq \alpha$.
\end{theorem}

Algorithms for probably learning GMMs have been well-studied by the Theoretical Computer Science community.
Our algorithm accomplishing \Cref{thm:informal:GMM-estimation} follows the well-studied cluster-then-learn paradigm~\citep{dasgupta1999mixtures}, which requires some form of separation condition of the underlying distribution.
See \Cref{sec:related-works} for more details.

Now, for a ``well-behaved'' loss function, e.g., Lipschitz, we can provably approximate the loss of the original dataset for the purposes of downstream training.
Moreover, our parameter estimation algorithm yields a conditional generation algorithm for labeled data by running the estimation algorithm for each class.
\begin{theorem}[Informal Downstream Training; See \Cref{thm:wasserstein-generation}]
\label{thm:inf:cond:gen}
    Let $\eps, \delta, \alpha, \beta\in(0,1)$
    and $f$ be a $(\lambda,z)$-H\"{o}lder continuous loss function for $z\in [1, 2]$.
    Suppose $Z = (X, Y)$ is a joint feature-label distribution for $Y\in [c]$ where
    each conditional distribution $(X\mid Y=y)\sim \calD_{\GMM}^{(y)}$ is a well-separated Gaussian mixture model.
    Given $n$ samples from each conditional distribution
    for $n=\poly(k, d, \nicefrac1{\alpha}, \nicefrac1\eps, \log(\nicefrac1\beta), \log(\nicefrac1\delta))$,
    there is an $(\eps,\delta)$-DP algorithm that outputs a distribution $\tilde Z = (\tilde X, Y)$
    such that with probability $1-\beta$,
    $
        \EEx{Z}{f(Z)}
        \le \EEx{\tilde Z}{f(\tilde Z)} + \lambda\cdot \alpha\,.
    $
\end{theorem}

We also show that our algorithm satisfies $(\eps, \delta)$-DP and can be implemented in near-linear time.
\begin{theorem}[Informal; See \Cref{thm:alg-privacy,thm:running-time-formal}]
\label{thm:inf:priv:scala}
    Let $(\eps, \delta)\in (0, 1)$,
    $n$ be the number of input images,
    $T$ be the maximum runtime of \Encode and \Decode on a single input,
    and $d$ the embedding dimension.
    \Cref{alg:dp-synthetic-generation} is $(\eps, \delta)$-DP
    and can be implemented in $\tilde O(n(d+T)\cdot \poly(\nicefrac1\eps, \log(\nicefrac1\delta)))$ time.
\end{theorem}

We implement and test our framework on standard benchmark datasets from DP classification and synthetic data literature~\citep{LinGKNY24,ghalebikesabi2023diffusion,de2022unlocking,torkzadehmahani2019DPCGAN}
at the same privacy levels as the state-of-the-art (SOTA) DP classification~\cite{de2022unlocking}.
While our theoretical analysis hinges on separability conditions,
we find that our method empirically yields strong downstream classification accuracy regardless. 
Specifically,
we train a simple two-layer neural network on DP synthetic embeddings
and compare its accuracy against all DP training methods,
including those that do not use synthetic data.
We obtain SOTA classification accuracy on standard datasets in the DP synthetic data literature (See \Cref{sec:exp}).

Note that one would expect training via DP synthetic data generation to achieve worse performance than direct training via DP-SGD,
as the former is a more general task.
This belief is supported by previous work on DP synthetic image generation~\cite{LinGKNY24}.
Thus it is very surprising that we can achieve comparable,
not to mention new SOTA DP classification results.

\subsection{Related Works}\label{sec:related-works}
There are many related works that are relevant to this paper.
We discuss the immediately related works
and defer the rest to \Cref{apx:related-works}.

\paragraph{DP synthetic data.}
Given a private dataset $D$, the goal is to privately generate a synthetic dataset which is statistically similar to $D$~\citep{LinGKNY24,xie2024text,hou2024PrEText,ghalebikesabi2023diffusion,harder2023pretrained,yue2023simple,kurakin2023harnessing,torfi2022medical,mattern2022secure,rosenblatt2020differentially,torkzadehmahani2019DPCGAN,abay2018privacy}.
See \citep{hu2024SoK,chen2024unified} and references therein for a survey of recent developments.
One related line of work on DP synthetic data given only \emph{API-access} to foundation models~\cite{LinGKNY24,xie2024text} also develops training-free methods that leverage pre-trained embeddings.
However,
they only do so in the context of establishing a measure of difference between a candidate synthetic dataset and the true sensitive dataset.
We further leverage the power of pre-trained embeddings by clustering together similar data points in the embedding space and modeling each cluster using a Gaussian distribution.

\paragraph{DP clustering.}
DP $k$-means clustering seeks to identify groups of similar data points
while ensuring the output is not overly sensitive to the value of any particular entry.
\cite{SuCLBJ16,SuCLBLJ17,huang2018optimal,lv2019optimizing}.
A particularly relevant line of work is that on scalable DP clustering algorithms which terminate in near-linear running time~\citep{cohen2022near,cohen2022scalable}.

\paragraph{(DP) Gaussian mixture models.}
Mixture models were introduced by \citet{pearson1894contributions} for modeling the presence of subpopulations.
The most popular algorithm for estimating GMMs in practice is a heuristic called Expectation-Maximization (EM)~\citep{dempster1977maximum}.
Unfortunately,
EM does not provably learn GMMs.
In a seminal paper,
\citet{dasgupta1999mixtures} designed the first (efficient) clustering-based algorithm that provably learns a GMM under separation conditions similar to ours.
The cluster-then-learn scheme introduced by \citet{dasgupta1999mixtures} led to follow-up works~\citep{dasgupta2000two-round,arora2005learning,vempala2004spectral} following said scheme that shaved the degree of separation needed.
Departing from clustering-based techniques,
\citet{kalai2010efficiently,moitra2010settling} developed sophisticated algorithms for learning GMMs without any separation conditions.
Unlike clustering-based algorithms,
These algorithms have polynomial dependence on relevant parameters except $k$ (the number of components),
which is unfortunately necessary in the absence of separation conditions.
See e.g. \citep{moitra2018algorithmic} for a more detailed survey of prior algorithmic developments.

In general,
the covariance matrices within each component of a GMM (\defref{def:GMM}) can be arbitrary.
However,
various restrictions of the covariance structure have been studied and applied across various fields,
including spherical covariances~\cite{hsu2013learning},
diagonal covariances~\cite{reynolds2009gaussian},
and tied covariances~\cite{greenspan2006constrained}.

In the differential privacy community,
prior works have studied the task of privately learning a GMM under various assumptions~\citep{ParkFCW17,KamathSSU19}.
See \citet{arbas2023polynomial} and references therein for a more comprehensive history.
However,
practical implementations have so far been underexplored.
Our simple algorithm for privately fitting a GMM based on the more well-studied task of DP clustering may be of interest beyond synthetic data generation.

\subsection{Preliminaries}
We defer standard preliminaries to \Cref{apx:prelim}.

\paragraph{Notation.}
We write $d$ to denote the ambient dimension,
$\eps, \delta$ to denote the approximate-DP parameters,
and $\alpha,\beta$ to denote the accuracy and failure probability parameters.
We use $k$ to denote the number of clusters or components for $k$-means or Gaussian mixture models,
respectively.
Typically,
we use $\mu, \Sigma$ to denote the mean and covariance of a distribution.

\section{Overview of Techniques \& Utility Analysis}
\seclab{sec:methods}
\begin{algorithm}[ht]
   \caption{DP Synthetic Generation}
   \label{alg:dp-synthetic-generation}
    \begin{algorithmic}[1]
    \State {\bfseries Input:} 
        data $D = \set{\bx_1, \dots, \bx_n}$, 
        privacy parameters $\eps, \delta$, 
        number of clusters $k$,
        number of generations $m$
    \vspace{0.25em}

    \State $D_\Embedding\gets \set{\Encode(\bx): \bx\in D}$ \label{line:embedding}
    \vspace{0.25em}
    
    \State $(\bc_1, n_1), \dots, (\bc_k, n_k)\gets \DPCluster(D_\Embedding, \nicefrac\eps5, \nicefrac\delta5)$ 
    \vspace{0.25em}
    
    \For{$j=1,\dots, k$}
        \State $D_j \gets \{ \bx\in D_\Embedding: \bc_j=\argmin_{\bc=\bc_1, \dots, \bc_k}\norm{\bx-\bc}_2 \}$
        \State $\mathbf\mu_j \gets \DPMean(D_j, \nicefrac\eps5, \nicefrac\delta5)$
        \State $\Sigma_j \gets \DPCovariance(D_j, \nicefrac\eps5, \nicefrac\delta5)$
        \State $p_j \gets n_j / \sum_{j=1}^k n_j$ 
    \EndFor
    \vspace{0.25em}

    \State $Z_\Embedding\gets \varnothing$
    \For{$\ell=1, \dots, m$}
        \State $j\sim [k]$ with probability $p_j$
        \State $\bz_\ell\sim \calN(\mathbf\mu_j, \Sigma_j)$
        \State $Z_\Embedding\gets Z_\Embedding\cup \set{z_\ell}$
    \EndFor
    \vspace{0.25em}

    \State $Z \gets \DPFilterEmbedding(Z_\Embedding, D, \nicefrac\eps5, \nicefrac\delta5)$
    \State \textbf{yield} $Z$
    \vspace{0.25em}
    
    \State $Z_\Image\gets \set{\Decode(\bz): \bz\in Z}$
    \vspace{0.25em}

    \State \textbf{yield} $\DPFilterImage(Z_\Image, D, \nicefrac\eps5, \nicefrac\delta5)$
\end{algorithmic}
\end{algorithm}

In this section, we provide the theoretical guarantees for our training-free pipeline. 
We describe our procedures and the corresponding analysis for \emph{unconditional} generation. 
That is, there is no notion of a label for the dataset or equivalently, all the labels of the dataset are assumed to be the same. 
We remark that this is without loss of generality because for the case of \emph{conditional} generation, it suffices to repeat the procedure and analysis in parallel for each class in the training set.
We provide pseudocode for our method in \Cref{alg:dp-synthetic-generation}.

\subsection{Subroutines}
\paragraph{Encoders \& decoders.}
Our utility analysis relies on the loss function being H\"{o}lder continuous over the input space.
embedding space.
While this may seem to be a strong assumption,
it has been experimentally verified to hold for certain embeddings~\cite{raffel2020exploring,devlin2019bert}.

Thus,
to privately train a classifier
by training on DP synthetic embeddings,
we assume there is a publicly available encoder module \Encode 
that takes a $(C\times W\times H)$ image 
and outputs a vector $x\in\bbR^d$.
Here $C$ is the number of image channels
and $W, H$ are the width and height of the input image.

If in addition, we wish to generate DP synthetic images,
we assume access to a decoder module \Decode that takes a vector $x\in \bbR$ and maps it back to an image,
possibly of different dimensions $(C'\times W'\times H')$. 

\paragraph{Filtering embeddings \& images.}
Similar to any (not necessarily private) data generation process,
our method may occasionally generate an embedding or an image that is a poor representation of the underlying sensitive data.
Thus, our algorithm optionally supports filtering at the embedding and image level, where we discard some of the generated embeddings or images based on some rules \DPFilterEmbedding, \DPFilterImage.
Similar to \citet{LinGKNY24,xie2024text,hou2024PrEText},
we allow the filtering to depend on private data.

\subsection{Synthetic Embeddings}
Our full algorithm for generating synthetic data is presented in \Cref{alg:dp-synthetic-generation}.
While the pseudocode includes the optional image generation step, it suffices to stop before the decoding step for the purpose of training a classifier on DP synthetic embeddings.
The rest of this section delves into some details and analysis of \Cref{alg:dp-synthetic-generation}.

\paragraph{Encoding images.}
We use a variant of the pre-trained CLIP \citep{RadfordKHRGASAM21} image encoder to encode each training and test image into $768$-dimensional embeddings.
In particular,
we use \texttt{CLIPImageProcessor}\footnote{\url{https://huggingface.co/docs/transformers/v4.48.0/en/model_doc/clip\#transformers.CLIPImageProcessor}} and \texttt{CLIPVisionModelWithProjection}\footnote{\url{https://huggingface.co/docs/transformers/v4.48.0/en/model_doc/clip\#transformers.CLIPVisionModelWithProjection}}.
Both the models and model weights\footnote{\url{https://huggingface.co/diffusers/stable-diffusion-2-1-unclip-i2i-l/tree/main}} are publicly available on HuggingFace.
Note that there are no private operations in this step.

\paragraph{Privately learning a GMM.}
Next,
we privately fit a $k$-Gaussian Mixture Model ($k$-GMM) on the embeddings produced by the previous step.
This comprises of two steps:
learning a partition of the dataset using a private $k$-means algorithm
and privately estimating the intra-cluster means/covariances given these private centers.
We analyze both steps in \Cref{apx:learn-gmm}.

Intuitively,
assuming the data embeddings were generated from a $k$-GMM,
a reasonable approximate $k$-means solution must place a center close to each cluster.
Then,
assuming the components are sufficiently well-separated,
it should be the case that each output center is also well-separated 
and hence we can ``classify'' points by nearest center.
We formalize this intuition in \Cref{thm:k-means-learns-GMM}.

Then,
these $k$ centers induce a partition of the dataset,
where a point belongs to the $i$-th partition
if its closest center is the $i$-th center.
Assuming we managed to capture only points from the $i$-th component in the $i$-th partition,
we can estimate the parameters of the $i$-th component using any algorithm for Gaussian estimation.
This is made formal in \Cref{thm:dp-GMM-estimation}.

In our experiments,
we use the practical DP $k$-means algorithm by \citet{chang2021clustering}
to privately compute $k$ centers.
Note that the number of clusters $k$ is a tuned hyperparameter.
We also output a noisy count of the number of elements within each partition (cf. \Cref{alg:dp-synthetic-generation}).

For the second step,
we estimate the intra-cluster means and covariances by clipping and adding appropriate Gaussian noise.
There are many variations of restricted covariance models within GMMs (see \Cref{sec:related-works})
and we empirically noticed that diagonal covariances yield the best performance.

\paragraph{Private synthetic embedding generation.}
Given the private $k$-GMM,
we can then generate an unlimited number of synthetic embeddings 
simply by sampling from the GMM. 
This does not incur additional privacy loss as it is post-processing.

Optionally,
we prune the generated point using noisy votes from original training data,
similar to a single iteration of Private Evolution \citep{LinGKNY24,xie2024text}.
That is,
each original embedding point votes for its nearest neighbor in the generated embeddings.
After adding an appropriate amount of noise to the count to preserve privacy,
we keep a generated embedding only if its noisy vote is above a certain threshold.
This threshold is a hyperparameter.

\paragraph{Training a classifier on synthetic embeddings.}
Given a dataset of synthetic embeddings,
our goal is to train a model by minimizing an appropriate well-behaved loss function over the synthetic embeddings.
We analyze this step in \Cref{apx:utility}.

As mentioned before,
since we can generate as many synthetic embeddings as we want, 
\Cref{eq:classification-loss-decomposition} shows that the proxy error arising from training on synthetic embeddings should be dominated by the estimation error.
We translate the parameter estimation error to a distributional bound in Wasserstein distance between GMMs,
which implies a bound on the proxy error for H\"older continuous functions.
This is quantified in \Cref{thm:wasserstein-generation}.

Experimentally,
we train a simple two-layer neural network on the synthetic embeddings
and test its accuracy on the \emph{original} test set embeddings.
As remarked earlier,
using non-private training techniques does not incur any private loss,
as the synthetic embedding generation process is differentially private.
We achieve SOTA accuracy on CIFAR-10~\citep{alex2009cifar} and CAMELYON17~\citep{bandi2019camelyon17}.
We also achieve comparable accuracy on the more challenging CIFAR-100~\citep{alex2009cifar} dataset.
See \Cref{sec:embedding-training} for more details.

\subsection{Synthetic Images}
The above already suffices to train a private classifier.
If we wish to also generate images,
it can be obtained with the help of a decoder module.

\paragraph{Decoding Embeddings into Images.}
We use StableUnCLIP,
a stable diffusion model fine-tuned on CLIP embeddings~\citep{Rombach_2022_CVPR} to decode CLIP embeddings into $768\times 768$ images.
Specifically,
we use the class \texttt{StableUnCLIPImg2ImgPipeline}\footnote{\url{https://huggingface.co/docs/diffusers/en/api/pipelines/stable_unclip\#diffusers.StableUnCLIPImg2ImgPipeline}}
with publicly available weights\footnote{\url{https://huggingface.co/diffusers/stable-diffusion-2-1-unclip-i2i-l/tree/main}} through HuggingFace.

Optionally,
we use NIQE \citep{mittal2013niqe} and PIQE \citep{venkatanath2015piqe,sheikh2005piqe} image quality filters to filter out the generated images that are too noisy.
Note that the two pruning strategies do not depend on the private data
and simply compute a ``quality'' score given an input image.

\paragraph{Training a Classifier on Synthetic Images.}
We use the decoded images to fine-tune a torchvision \citep{torchvision2016} implementation of the ResNet50 model \citep{he2016resnet} that was pre-trained on Imagenet \citep{deng2009imagenet}.
The only modification is to change the final classification layer to match the number of classes for the dataset in question.
Again,
we note that any non-private training method can be used to obtain a private classifier
since the training data is guaranteed to be differentially private.

\section{Experiments}
\label{sec:exp}
We begin by describing our experimental setup in \Cref{sec:setup}.

In \Cref{sec:embedding-training},
we compare the classification accuracy of a simple two-layer neural network trained on DP synthetic embeddings against SOTA private training methods on standard benchmark datasets.
We emphasize that we compare against all DP training methods,
including those that do not utilize synthetic data like DP-SGD.
Surprisingly, we achieve new SOTA results on CIFAR-10 and CAMELYON17 while obtaining competitive accuracy on the more challenging CIFAR-100 dataset.

In \Cref{sec:image-classification},
we demonstrate that our method is also able to generate useful private synthetic images
and compare the downstream classification accuracy of models trained on such images
against SOTA private synthetic image generation methods for CIFAR-10.
In particular,
our method achieves superior classification accuracy at all privacy budgets.

Finally,
\Cref{sec:epsilons-classification} presents a detailed comparison of our method against DP-SGD on various privacy budgets on CIFAR-10.

\subsection{Experimental Setup}\label{sec:setup}
\paragraph{Public data.} The CLIP embedding module~\citep{RadfordKHRGASAM21} was pre-trained on unspecified image-text pairs scoured from the internet.
We emphasize that our experiments on synthetic embeddings only use CLIP as public data.
Our results on synthetic images also requires a decoder.
Our decoder module is based on Stable Diffusion 2.1~\citep{Rombach_2022_CVPR},
which is trained on the LAION-5B~\citep{schuhmann2022laion5b} dataset
and fine-tuned to invert CLIP embeddings using the same dataset. 
For classification on synthetic images,
we fine-tune a model that was pre-trained on ImageNet~\citep{deng2009imagenet}.
We consider the above as publicly available data.

We selected CLIP embeddings because it is a crucial component of the only known general large pre-trained encoder-decoder model pair. 
We were unable to find other encoder-decoder pairs that generalize beyond the specific data sets upon which they were pre-trained.

\paragraph{Sensitive data.} 
We execute our DP synthetic generation pipeline on 
CIFAR-10, 
CIFAR-100~\citep{alex2009cifar}, 
and CAMELYON17~\citep{bandi2019camelyon17},
which we consider as private sensitive data.
The first two consist of natural images with 10 and 100 classes respectively,
and the CAMELYON17 is a medical dataset
for binary classification of breast cancer metastases.
We emphasize that these are standard benchmark datasets within the DP synthetic images literature~\cite{torkzadehmahani2019DPCGAN,ghalebikesabi2023diffusion,LinGKNY24}.

\paragraph{Hyperparameter tuning.} We performed hyperparameter search on the number of clusters
as well as the clipping radii for the generation algorithm
and followed the TorchVision formula for training\footnote{\url{https://pytorch.org/blog/how-to-train-state-of-the-art-models-using-torchvision-latest-primitives/}} without hyperparameter search.
Similar to other works on DP synthetic data~\citep{LinGKNY24,ghalebikesabi2023diffusion},
we do not account for hyperparameter search as part of the privacy budget.

\paragraph{Setup.}
Each training experiment is repeated for 3 runs and we report the mean accuracy and standard deviation.
Our experiments are performed using eight H100 GPUs (80GB memory each).
See \Cref{apx:experimental-details} for more setup details for our experiments.

\subsection{Private Training}\label{sec:embedding-training}
\begin{table}[h]
    \caption{Private training classification accuracies for various data sets.
    DP-SGD results taken from respective papers. $\delta=10^{-5}$ for all experiments.}
    \label{tab:embedding-training}
    \begin{center}
    \begin{small}
    \begin{sc}
    \begin{tabular}{lcccc}
    \toprule
    Dataset & $\eps$ & Ours & \makecell[c]{SOTA\\(DP)} & \makecell[c]{\textcolor{gray}{SOTA}\\\textcolor{gray}{(non-DP)}} \\
    \midrule
    cifar-10 & 8 & $97.0 \pm 0.01$ & 96.6 & \textcolor{gray}{99.5} \\
    cifar-100 & 8 & $80.5 \pm 0.177$ & 81.8 & \textcolor{gray}{96.1} \\
    camelyon17 & 10 & $93.1 \pm 0.067$ & 91.1 & \textcolor{gray}{95.7} \\
    \bottomrule
    \end{tabular}
    \end{sc}
    \end{small}
    \end{center}
    \vskip -0.1in
\end{table}

We compare the downstream classification of a simple two-layer neural network trained on private synthetic embeddings
against the classification accuracy of all other private training methods.
See \Cref{apx:embedding-classifier} for details of the two-layer neural network.

DP-finetuning~\citep{de2022unlocking} achieves the current SOTA on CIFAR-10, CIFAR-100
and DP-Diffusion~\citep{ghalebikesabi2023diffusion} achieves the current SOTA on CAMELYON17.
We also display the non-DP SOTA results,
as reported by \citet{dosovitskiy2021image,foret2021sharpness,lee2019automatic}.
For the above datasets,
we generate the same number of synthetic embeddings as the original training splits
and train a simple two-layer neural network from scratch on said embeddings.
We then test on embeddings of the \emph{original} (non-synthetic) test set.
\Cref{tab:embedding-training} shows that our method achieves an improvement on the SOTA for CIFAR-10 and CAMELYON17
at the same privacy budgets as prior works.

We emphasize that this runs contrary to conventional beliefs~\citep{LinGKNY24},
as DP synthetic data is more general-purpose than training via DP-SGD,
which is optimizing for a single task.

Our results suggest that even on private datasets with significant distribution shift from the encoder training data,
training on synthetic embeddings can yield classifiers with strong privacy-to-utility tradeoffs.

\subsection{Private Synthetic Images}\label{sec:image-classification}
Next,
we compare the downstream classification accuracy of classifiers trained on synthetic images generated by decoding embeddings
against other baseline DP synthetic image techniques on CIFAR-10.
Private Evolution~\citep{LinGKNY24} achieves the current SOTA on lower values of $\eps$
while DP-Diffusion~\citep{ghalebikesabi2023diffusion} achives the current SOTA on higher values of $\eps$.

We fine-tune a ResNet50~\citep{he2016resnet} classifier pre-trained on ImageNet~\citep{deng2009imagenet}
using 50,000 DP synthetic images
and test its accuracy on the original (non-synthetic) CIFAR-10 test set.
\Cref{fig:synthetic-images} compares the results across various levels of $\eps$.
We note that \citet{harder2023pretrained} achieve an accuracy of $51\%$ (not shown).

\begin{figure}[h]
    \centering
    \includegraphics[width=0.5\linewidth]{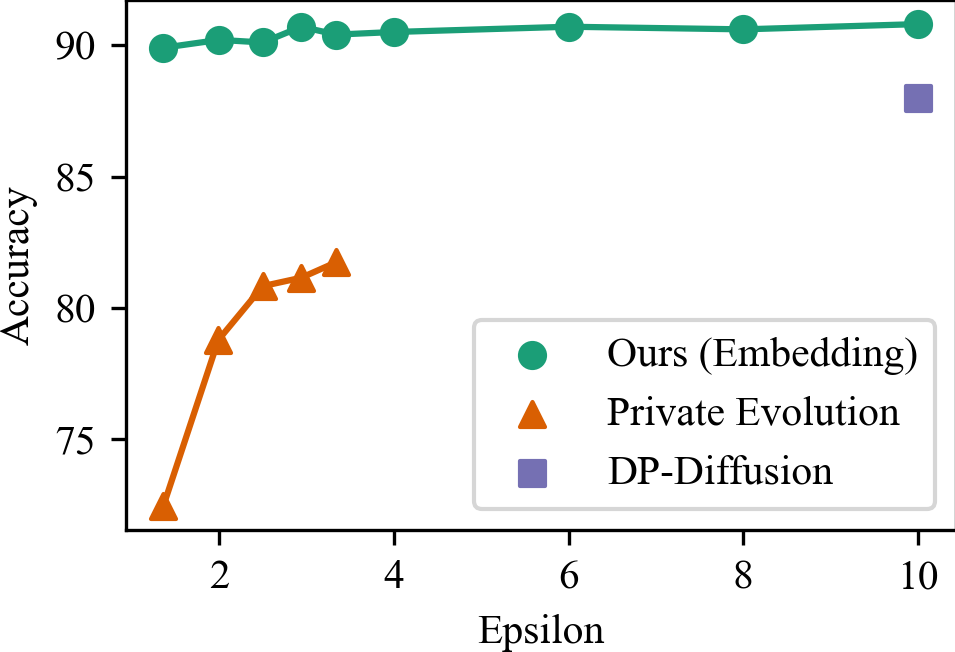}
    \caption{Downstream classification accuracy on 50,000 generated CIFAR-10 images
    at various levels of $\eps$ and $\delta=10^{-5}$.
    We report the baseline accuracies at every available level of $\eps$, exactly as stated in their respective papers.
    }
    \label{fig:synthetic-images}
\end{figure}

The above shows that when the decoder module is pre-trained on similar data to the private dataset,
our method can achieve strong utility at lower privacy budgets.

\subsection{Privacy-Utility Tradeoffs}\label{sec:epsilons-classification}
Finally,
we consider the privacy-utility tradeoffs of our methods 
by examining the performance
at varying levels of $\eps$.

\begin{table}[h]
    \caption{Downstream classification accuracies when trained on 50,000 synthetic embeddings or images
    and tested on CIFAR-10. $\delta=10^{-5}$ for all experiments.}
    \label{tab:privacy-levels}
    \begin{center}
    \begin{small}
    \begin{sc}
    \begin{tabular}{lcccr}
    \toprule
    $\eps$ & \makecell[c]{DP-SGD\\(fine-tuning)} & \makecell[c]{Ours\\(embeddings)} & \makecell[c]{Ours\\(images)} \\
    \midrule
    1 & 94.8 & $96.6 \pm 0.074$ & $89.7 \pm 0.143$ \\
    2 & 95.4 & $96.8 \pm 0.087$ & $90.2\pm0.129$ \\
    4 & 96.1 & $96.9 \pm 0.064
$ & $90.5 \pm 0.044
$ \\
    8 & 96.6 & $97.0 \pm 0.01$ & $90.6 \pm 0.054
$ \\
    \bottomrule
    \end{tabular}
    \end{sc}
    \end{small}
    \end{center}
    \vskip -0.1in
\end{table}

Our first comparison is quantitative,
and we examine the downstream classification accuracy for both synthetic embeddings and actual images for the CIFAR-10 dataset.
For reference,
we consider DP-finetuning~\citep{de2022unlocking}.
\Cref{tab:privacy-levels} shows that our classifier trained on synthetic embeddings consistently outperforms the one trained by DP-finetuning at various levels of $\eps$.

Next,
we qualitatively compare of the images generated at different privacy levels.
See \Cref{apx:cifar10-synthetic-images} for examples of synthetic CIFAR-10 images at various levels of $\eps$.
Interestingly,
while the classification accuracy does not significantly decrease as $\eps$ varies,
the variance of the generated images noticeably increases as $\eps$ decreases.
Consider for example,
the $9$-th row of \Cref{fig:cifar10-eps8-apx,fig:cifar10-eps1-apx},
which displays synthetic images of boats at $\eps=8, 1$,
respectively.
In \Cref{fig:cifar10-eps8-apx},
each of the 10 images is recognizable as a boat.
However,
in \Cref{fig:cifar10-eps1-apx},
only 3 of the 10 images resemble some form of a boat
while the others in the row are abstract shapes.
Similar occurrences can be observed for the other classes

\section{Limitations \& Future Works}
\seclab{sec:limits}

\paragraph{DP clustering.}
An important subroutine in our generation pipeline is DP clustering.
Improved implementations of DP clustering can also improve our algorithm,
further motivating research on DP clustering.

\paragraph{Decoding.}
We were unable to find other encoder-decoder pairs that generalize beyond their training data, 
which necessitated the use of CLIP embeddings if we wish to generate images.
Progress on general-purpose encoder-decoder models or exploration of domain-specific encoder-decoder pairs
will broaden the applicability of our method.

\paragraph{Filtering.} The generated images were not carefully filtered,
and we used two simple filtering heuristics that are agnostic to the sensitive data.
Using more sophisticated methods~\citep{bandi2019NIMA,ke2021MUSIQ} may yield better performance.
Furthermore,
we can use existing image enhancement methods~\citep{qi2021comprehensive} to improve the quality of generated images.

\section{Conclusion}
Our work introduces a novel principled framework for private training and data generation by clustering embeddings, demonstrating significant improvements in privacy-utility tradeoffs compared to existing approaches in the DP synthetic data literature. 
Moreover, by leveraging DP synthetic embeddings, we achieve state-of-the-art classification accuracy on CIFAR-10 and CAMELYON17, highlighting the potential of our method in real-world applications. 
Our method offers a practical solution for privately training classifiers without exposing real data, which is particularly valuable in domains where data sharing is restricted.

\section*{Acknowledgments}
This work was partially completed while Felix Zhou was a student researcher at Google Research. 
Felix Zhou acknowledges the support of the Natural Sciences and Engineering Research Council of Canada (NSERC).
Samson Zhou is supported in part by NSF CCF-2335411.

\begingroup
\sloppy
\printbibliography
\endgroup

\clearpage
\appendix

\section{CIFAR-10 Synthetic Images}\label{apx:cifar10-synthetic-images}
In this section,
we show examples of synthetic CIFAR-10 images generated at various levels of $\eps$.
We randomly chose 10 images per class from each of the synthetic and original training sets
and display them side-by-side.
As expected,
there is a noticeable decrease in the fidelity of the synthetic images as $\eps$ decreases,
due to the increase in noise injected into the system.

\begin{figure}[H]
    \centering
    \includegraphics[width=\linewidth]{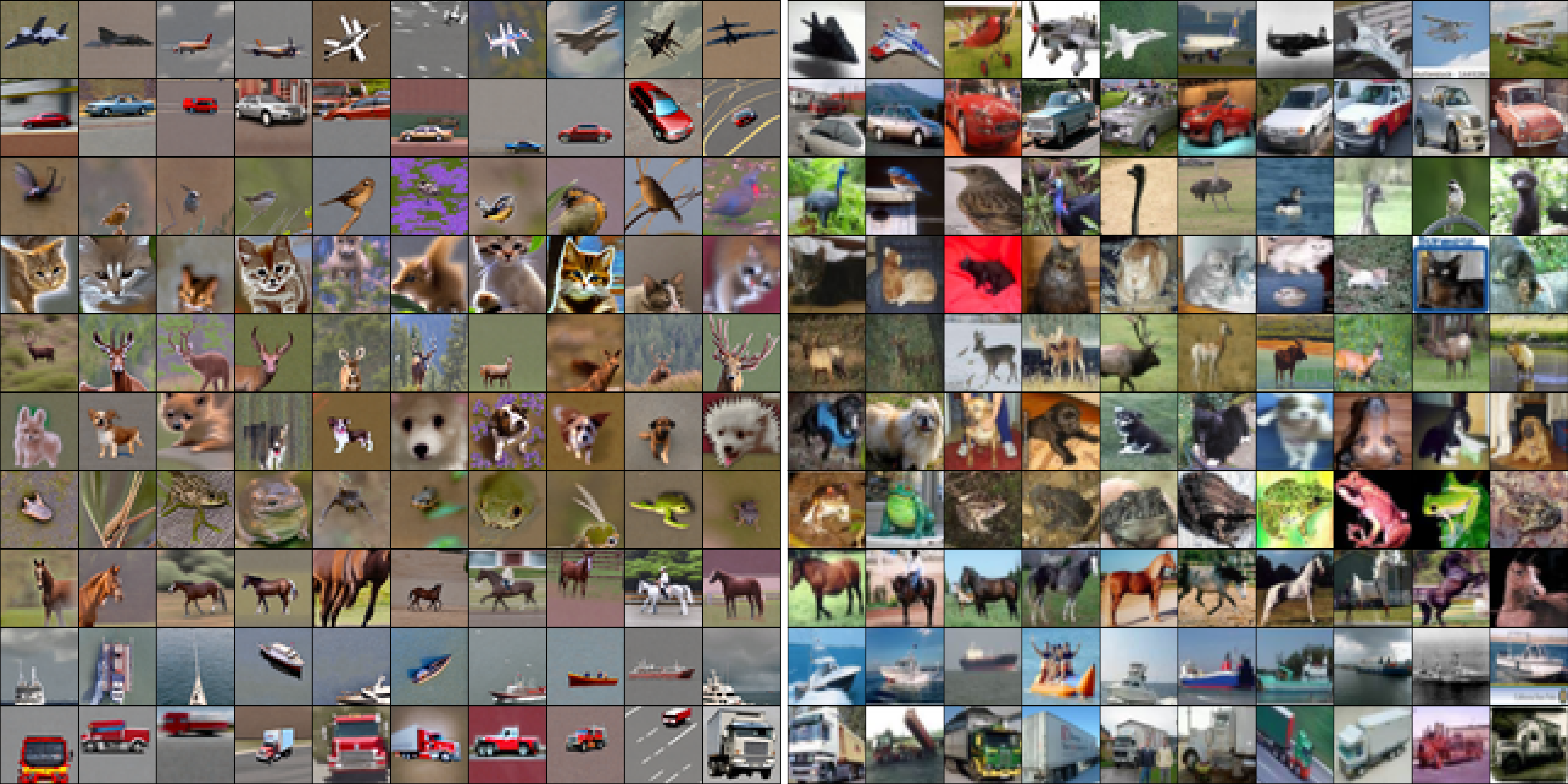}
    \caption{CIFAR-10 synthetic images at $\eps=8, \delta=10^{-5}$.
    Each row corresponds to a different class.
    The left-most columns are synthetic images while the right-most columns are original images.}
    \label{fig:cifar10-eps8-apx}
\end{figure}

\begin{figure}[H]
    \centering
    \includegraphics[width=\linewidth]{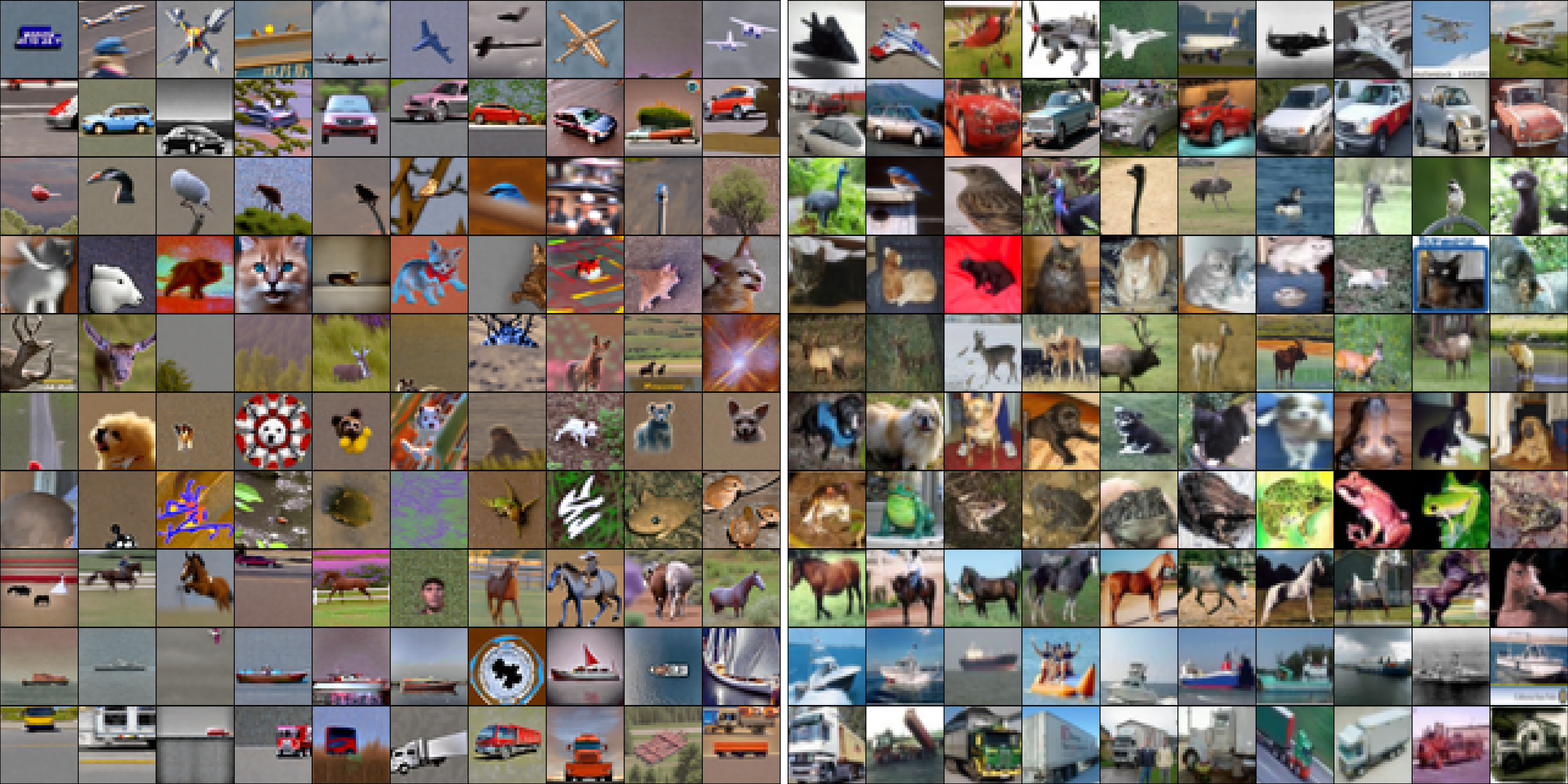}
    \caption{CIFAR-10 synthetic images at $\eps=4, \delta=10^{-5}$.
    Each row corresponds to a different class.
    The left-most columns are synthetic images while the right-most columns are original images.}
    \label{fig:cifar10-eps4-apx}
\end{figure}

\begin{figure}[H]
    \centering
    \includegraphics[width=\linewidth]{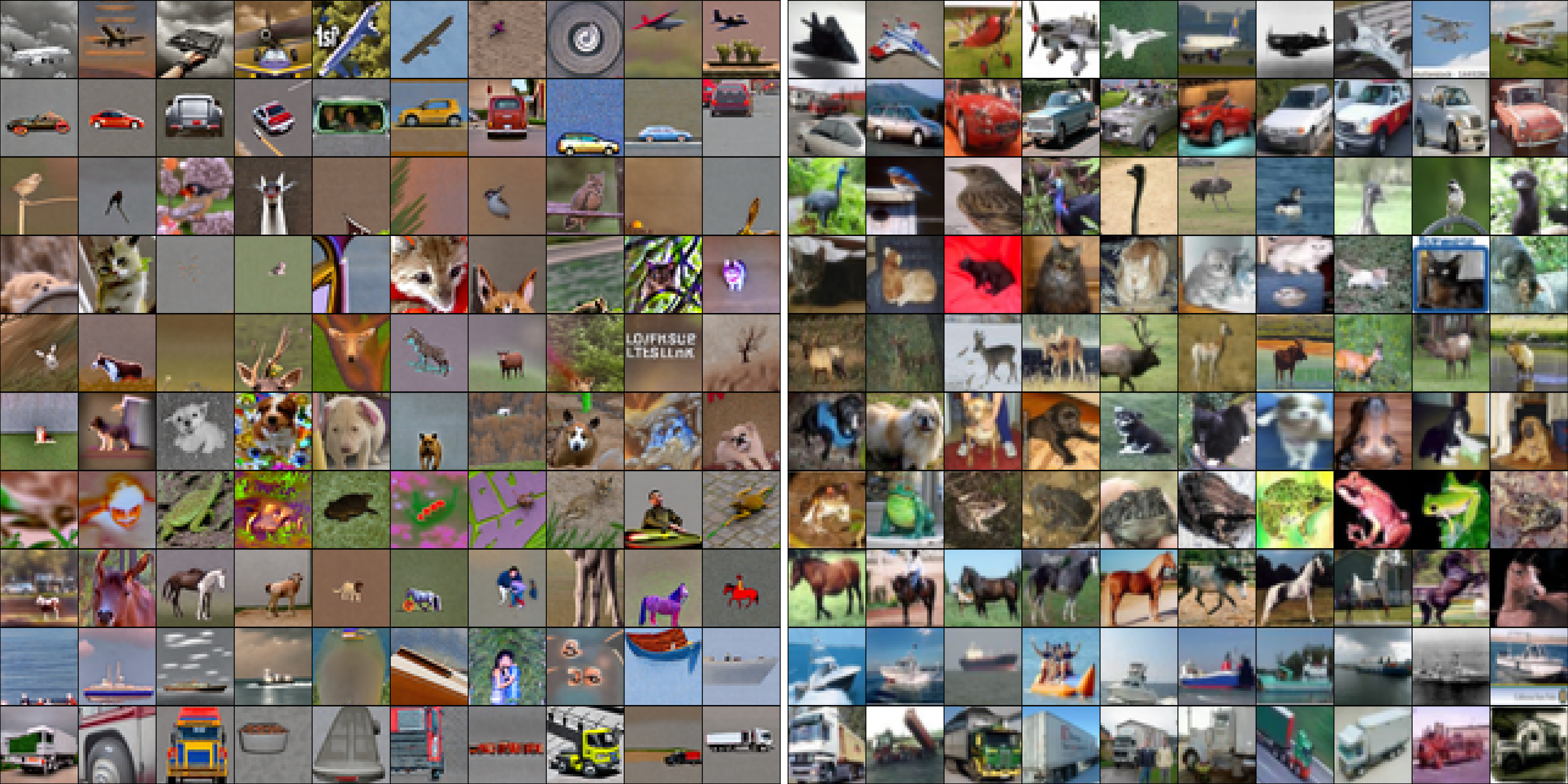}
    \caption{CIFAR-10 synthetic images at $\eps=2, \delta=10^{-5}$.
    Each row corresponds to a different class.
    The left-most columns are synthetic images while the right-most columns are original images.}
    \label{fig:cifar10-eps2-apx}
\end{figure}

\begin{figure}[H]
    \centering
    \includegraphics[width=\linewidth]{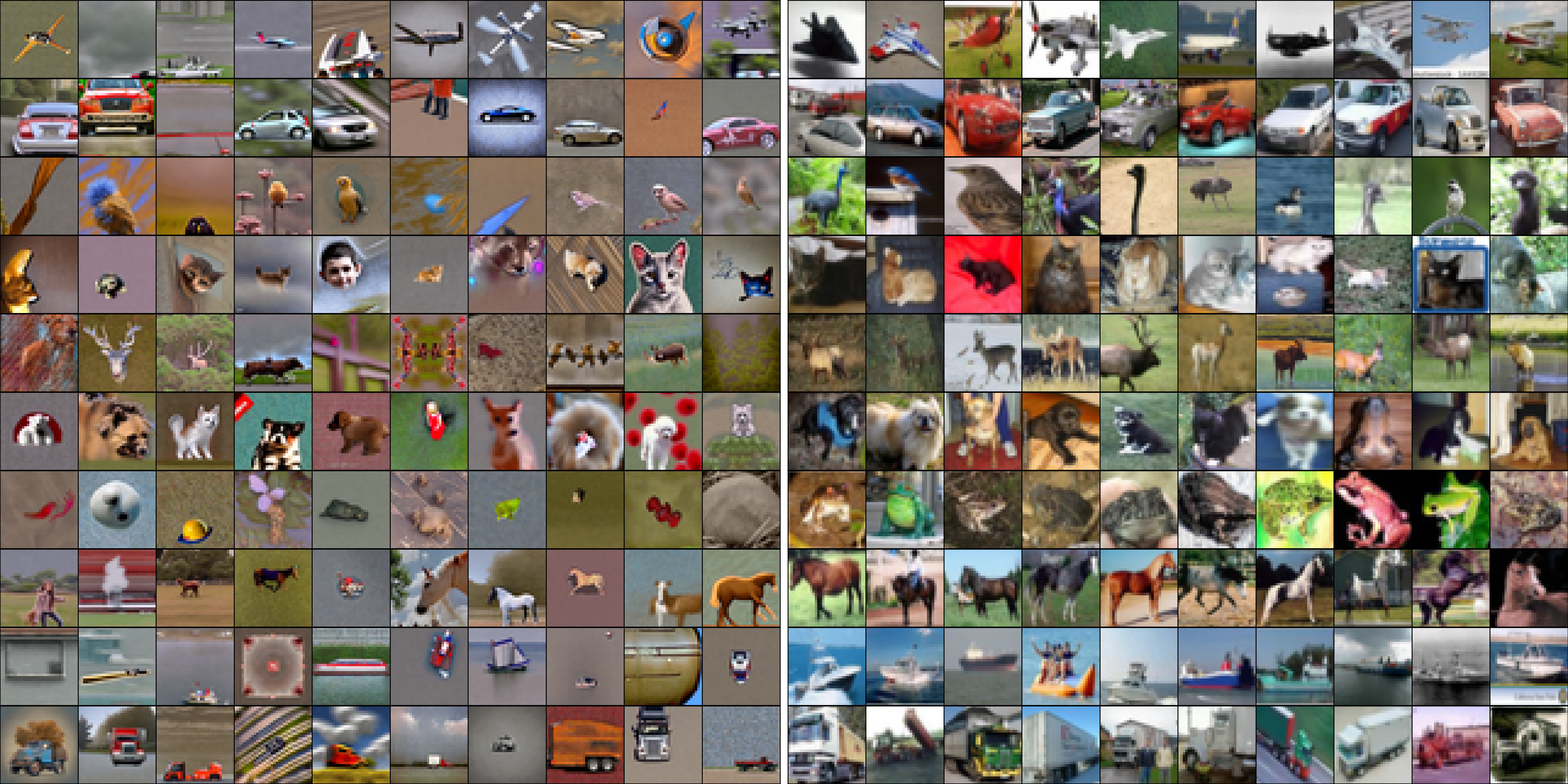}
    \caption{CIFAR-10 synthetic images at $(\eps=1, \delta=10^{-5})$.
    Each row corresponds to a different class.
    The left-most columns are synthetic images while the right-most columns are original images.}
    \label{fig:cifar10-eps1-apx}
\end{figure}

\section{Related Work}\label{apx:related-works}
The areas most related to our work are that of data selection, (non-private) synthetic data generation, 
and private fine-tuning.

\paragraph{Synthetic data generation.}
The seminal paper of \citet{GoodfellowPMXWOCB14} introduced Generative Adversarial Networks (GANs), which intuitively train the generator not to minimize the loss function for labels of individual images, but instead to fool a ``discriminator'' neural network that can tell how ``realistic'' the input seems. 
GANs have been widely used for generating synthetic data~\citep{RadfordMC15} across a wide range of applications, e.g., to privately create synthetic medical images to enhance CNN performance~\citep{Frid-AdarDKAGG18,SandfortYPS19} in healthcare, as well as for applications in data augmentation~\citep{AntoniouSE17}, facial recognition~\citep{YiLLL14a}, image super-resolution~\citep{LedigTHCCAATTWS17}, and text-to-image synthesis~\citep{ReedAYLSL16}. 

\paragraph{Private fine-tuning.}
Another relevant area is that of private fine-tuning, which has been used for tasks such as privately training large language models~\citep{YuNBGI0KLMWYZ22}. 
Private fine-tuning is the process of adapting a pre-trained machine learning model to a specific task using a sensitive dataset, while ensuring that individual data points in the dataset remain private. 
The main intuition behind private fine-tuning is to modify the model's training process with private techniques to prevent the model from compromising sensitive information, e.g., using methods such as DP-SGD~\citep{AbadiCGMMT016}, which adds noise to the gradients during training and clips them to control the influence of any single data point. 
Consequently, the resulting model retains useful task-specific knowledge while preserving differential privacy. 

Recently, \citet{LinGKNY24} observed that API-based solutions are becoming increasingly popular, in part due to the accessibility of these systems to users without ML-specific expertise. 
Thus, they view API providers as untrusted entities and proposed the Private Evolution algorithm for generating private synthetic data using black-box APIs of foundation models. 
Despite not having access to model weights and gradients, the key idea of Private Evolution is to iteratively utilize private samples to determine the most similar samples generated from the black-box model and ask the black-box models to generate more of those similar samples.

\paragraph{Data selection.}
Data selection aims to select the most ``important'' data points to label from a pool of unlabeled examples, thereby maximizing accuracy with a smaller labeled dataset, reducing the cost and effort of labeling while still achieving high model performance. 
While a universally optimal data selection strategy is not achievable~\citep{Dasgupta04}, several heuristics~\citep{Settles09,RenXCHLGCW22} have proven to be effective in practice. 
Nevertheless, these traditional data selection approaches all focus on selecting the most important data points to improve model performance. 
Indeed, data selection inherently and necessarily reveals crucial structural information about these data points that is subsequently used for generalization. 

Data selection for machine learning is a well-studied approach for improving model efficiency, robustness, and generalization, particularly in the context of deep learning and neural networks. 
Perhaps the most relevant technique to our line of study is importance sampling, where training emphasizes samples that contribute most to the learning objective, enhancing gradient efficiency and thus convergence rates~\citep{KatharopoulosF18}. 
Similarly, data pruning methods seek to remove redundant or less informative samples in the training data, while retaining critical patterns, thereby lowering computational costs without sacrificing accuracy~\citep{MolchanovTKAK17}. 
For CNNs, where spatial structure in data is crucial, hard example mining~\citep{ShrivastavaGG16} is often employed to focus training on misclassified or challenging samples, which helps the model learn more expressive features. 
Other data augmentation and selection methods, such as diversity-based selection, aim to include a wide range of spatial and contextual patterns, to reduce overfitting and improve robustness~\citep{ShortenK19}. 
On the other hand, task-specific methods, such as domain-aware selection in transfer learning, prioritize source domain samples similar to the target domain, enabling better feature transfer~\citep{PanY10}.

\section{Further Experimental Details}\label{apx:experimental-details}
\subsection{Datasets}
CIFAR10 consists of 60,000 $32\times 32$ natural images in 10 equal-sized classes.
The standard training split consists of 50,000 images.

CIFAR100 is similar and consists of 60,000 $32\times 32$ natural images in 100 equal-sized classes.
The training split consists again of 50,000 images.
This is a challenging dataset for DP synthetic data as each class has only 500 training images
and different classes can have very similar appearances.

Finally,
CAMELYON17 is a medical dataset
for classification of breast cancer metastases.
It consists of $96\times 96$ image patches of lymph node tissue from five different hospitals.
The label signifies whether at least one pixel in the center $32\times 32$ pixels has been identified as a tumor cell.
CAMELYON17 is part of the WILDS~\citep{koh2021wilds} leaderboard as a domain generalization task: 
The training split contains 302,436 images from three different hospitals whereas the test split
contain 85,054 images from a fourth and fifth hospital.

\subsection{Generation Details}\label{apx:generation-hyperparams}
We perform a gridsearch over 
the number of $k$ of GMM components,
and the intra-cluster clipping radius for covariance estimation.
We choose $k\in \set{1,2,4,8,16}$ and the clipping radius between $\set{2.0, 4.0, 6.0, 8.0, 10.0}$.

We tried various covariance models of GMMs
and found that diagonal covariances yielded the best performance,
as spherical covariance models do not capture enough of the intra-cluster data
and the noise needed to privately estimate a full covariance matrix overwhelms the signal from the data.

When pruning embeddings,
we discard all generated embeddings that do not have a noisy vote of at least $6.0$.
If the number of images per class is $m$,
we generate $6m$ synthetic embeddings and find that approximately $2m$ embeddings survive.

When decoding images,
we process embeddings in batches of $16$ per GPU and generate $2$ images per embedding.
We discard all images whose NIQE~\citep{mittal2013niqe} or PIQE~\citep{sheikh2005piqe,venkatanath2015piqe} score falls below $20.0$.
We found that approximately $m$ images per class survive this process.

\subsection{Embedding Classification Details}\label{apx:embedding-classifier}
We train a simple two-layer neural network with a $128$-dimensional hidden layer,
batch normalization,
and dropout probability $0.5$.

We train on a single GPU with a batch size of $512$ for $50$ epochs.
We use SGD with an initial learning rate of $10^{-3}$ and cosine annealing learning rate scheduler.
For further regularization,
we set label smoothing to $0.2$ and weight decay to $10^{-4}$.

Before training,
if there are more than $n$ synthetic embeddings where $n$ is the size of the input sensitive training set,
then we select a random subset of size $n$.
This is to maintain fair comparison against private training baselines that do not use synthetic data
and therefore only have access to $n$ training points.

\subsection{Image Classification Details}\label{apx:training-hyperparams}
We fine-tune the torchvision implementation of ResNet50~\citep{he2016resnet} which was pre-trained on ImageNet~\citep{he2016resnet} with pre-trained weights publicly available from torchvision~\citep{torchvision2016}.

We train on a cluster of eight H100 GPUs (80GB memory each) with a batch size of $256$ per GPU for $30$ epochs.
As mentioned,
we follow the torchvision training recipe.\footnote{\url{https://pytorch.org/blog/how-to-train-state-of-the-art-models-using-torchvision-latest-primitives/}}
The only change is the initial learning rate of $10^{-2}$
and a pre-processing step where we encode and decode the test set (without using it in the training process).
We believe this last step improves test accuracy since the decoder creates some distributional shift between the synthetic embeddings.
Encoding and decoding the test set ensures the same shift is applied to the test set.

Before training,
we again restrict the synthetic training set to the same size as the original non-synth training set for fair comparison against other private synthetic data baselines.

\section{Deferred Preliminaries}\label{apx:prelim}
\begin{definition}[Gaussian Mixture Model (GMM); see e.g.\ \cite{reynolds2009gaussian}]
\deflab{def:GMM}
    A Gaussian Mixture Model (GMM) is a parametric probability density function represented as a weighted sum of Gaussian
    component densities
    $
        \sum_{i=1}^{k} w_i\cdot \calN(\mu_i, \Sigma_i)\,
    $
    where $w\geq 0$ satisfies $\sum_i w_i = 1$
\end{definition}

We also require our loss function to be well-behaved so that small perturbations in the input space do not result in large perturbations in the label space. 
Lipschitz continuity and its generalization to H\"{o}lder continuity are standard assumptions for this purpose (see e.g., \citet{AxiotisCHJMSWW24}). 
\begin{definition}[H\"{o}lder continuity]
    We say a function $f:\calX\times\calY\to\mathbb{R}$ is $(z,\lambda)$-H\"{o}lder continuous if for all $(\bx,y),(\bx',y)\in\calX\times\calY$, $|f(\bx, y)-f(\bx', y)|\le\lambda\|\bx-\bx'\|_2^z$.
\end{definition}

\subsection{Differential Privacy}
We first recall the following preliminaries from differential privacy. 
\begin{definition}[Differential privacy; \citet{DworkMNS06}]
    \deflab{def:dp}
    Given $\eps>0$ and $\delta\in(0,1)$, a randomized algorithm \mbox{$\calA:\calX^n\to\calR$} is $(\eps,\delta)$-differentially private if, 
    for every pair of neighboring datasets $D, D'\in\calX^n$ that differ by a single entry
    and for all subsets $U\subseteq\calR$ of the output space $\calR$,
    \[\PPr{\calA(D)\in U}\le e^{\eps}\cdot\PPr{\calA(D')\in U}+\delta\,.\]
\end{definition}

\begin{definition}[Laplace distribution]
A random variable $x$ follows the \emph{Laplace distribution} with mean $\mu$ and scale parameter $b>0$, denoted as $x\sim\Lap(\mu,b)$, if its probability density function is given by $\frac{1}{2b}\exp\left(-\frac{|x-\mu|}{b}\right)$. 
We use $x\sim\Lap(b)$ to denote $x\sim\Lap(0,b)$. 
\end{definition}
A common method to ensure differential privacy involves adding Laplacian noise, with scale parameter proportional to the following notion:
\begin{definition}[$L_1$ sensitivity]
    Let $x\sim y$ denote neighboring databases that differ by a single entry.
    The $\ell_1$ sensitivity of a function $f$ is defined by 
    \[\Delta_f=\max_{x, y: x\sim y}\|f(x)-f(y)\|_1.\]
\end{definition}
Informally, the $L_1$ sensitivity of a function is the largest amount that a single entry in a database can affect $f$. 
\begin{definition}[Laplace mechanism]
Given a function $f$, an input $x$, and a privacy parameter $\eps>0$, the Laplace mechanism outputs $f(x)+\eta$, where $\eta\sim\Lap\left(\frac{\Delta_f}{\eps}\right)$. 
\end{definition}
The Laplace mechanism is a fundamental technique for establishing differential privacy:
\begin{theorem}[\citet{DworkR14}]\thmlab{thm:laplace:dp}
The Laplace mechanism preserves $(\eps,0)$-differential privacy.
\end{theorem}

An important property of differential privacy is that performing computation on a privatized dataset cannot lose additional privacy:
\begin{theorem}[Post-processing of differential privacy; \citet{DworkR14}]
\thmlab{thm:dp:post}
Let $\calM$ be an $(\eps,\delta)$-differential private mechanism and $g$ be any arbitrary random mapping. 
Then $g(\calM(\cdot))$ is $(\eps,\delta)$-differentially private. 
\end{theorem}
Moreover,
multiple computations on a dataset incur privacy cost in a natural manner.
\begin{theorem}[Basic composition of differential privacy; \citet{DworkR14}]
\thmlab{thm:dp:comp}
Let $\calM_1$ and $\calM_2$ be $(\eps_1,\delta_1)$ and $(\eps_2,\delta_2)$-DP mechanism, respectively. 
Then the composition $(\calM_1(\cdot),\calM_2(\cdot))$ is $(\eps_1+\eps_2,\delta_1+\delta_2)$-differentially private. 
\end{theorem}

On the other hand,
executing a private mechanism on disjoint partitions of the same dataset
does not incur any additional privacy cost.
\begin{theorem}[Parallel composition of differential privacy; \citet{mcsherry2009parallelcomposition}]\thmlab{thm:dp:parallel-comp}
    Let $\calM$ be an $(\eps, \delta)$-DP mechanism
    and $D_1, \dots, D_k$ $k$-disjoint subsets of the dataset $D$.
    Then the mechanism that outputs $(\calM(D_1), \dots, \calM(D_k))$ is $(\eps, \delta)$-DP.
\end{theorem}

There are more sophisticated composition results, but \thmref{thm:dp:comp} and \thmref{thm:dp:parallel-comp} suffice for our purposes. 

\subsection{Concentration Inequalities}
\begin{theorem}[\citet{hanson1971bound}]\label{thm:gaussian-mean-concentration}
    Suppose $Y_i\sim_{i.i.d.} \calN(\mu, \Sigma)$ for $i\in [N]$.
    Then with probability $1-\beta$,
    \[
        \norm{\frac1N \sum_i Y_i - \mu}_2
        \leq \sqrt{\frac{\Tr(\Sigma)}{N}} + \sqrt{\frac{2\norm{\Sigma}_2\log(\nicefrac1\beta)}{N}}\,.
    \]
\end{theorem}

\begin{theorem}[Remark 5.40 in \citet{vershynin2010introduction}]\label{thm:gaussian-cov-concentration}
    Suppose $Y_i\sim_{i.i.d.} \calN(\mu, \Sigma)$ for $i\in [N]$ and let $\bar{Y}=\frac{1}{N}\sum_{i\in[N]}Y_i$. 
    Then with probability $1-\beta$,
    \[
        \norm{\frac1N \sum_i (Y_i - \bar Y)(Y_i - \bar Y)^\top - \Sigma}_2
        \leq O\left( \sqrt{\frac{d}{N}} + \sqrt{\frac{\log(\nicefrac1\beta)}{N}} \right) \norm{\Sigma}_2\,.
    \]
\end{theorem}

\begin{theorem}[Chernoff Bound; Corollary 4.6 in \citet{mitzenmacher2005probability}]\label{thm:chernoff bound}
Let $X_1, X_2, \ldots, X_n\in\{0,1\}$ be independent Bernoulli random variables and let $X = \sum_{i=1}^n X_i$. 
Then for any $\gamma\in(0,1)$,
\[
\PPr{|X -\Ex{X}| \ge \gamma\cdot\Ex{X}}\le 2 \exp\left( -\frac{\gamma^2\cdot\Ex{X}}{3} \right).
\]
\end{theorem}

\begin{theorem}[\citet{dvoretzky1956asymptotic,massart1990tight}]\label{thm:dkw-inequality}
    Let $F(x)$ denote the CDF function for an arbitrary distribution.
    For $x_1, \dots, x_n\sim_{i.i.d.} F$,
    write $F_n(x) = \frac1n\sum_{i=1}^n \ones\set{x=x_i}$ to denote
    the $n$-sample empirical CDF function $F_n$.
    It holds that
    \[
        \Pr\left[ \sup_x \abs{F_n(x) - F(x)} > \gamma \right]
        \leq 2\exp(-2n\gamma^2)\,.
    \]
\end{theorem}

\section{Private Gaussian Estimation}
The first sample-optimal DP Gaussian mean/covariance estimation algorithms were due to \citet{aden2021unbounded}.
However,
their algorithms require exponential running time.
Recent transformations from robust algorithms to private algorithms 
obtained the same optimal sample complexities for mean estimation~\citep{hopkins2022efficient}
and covariance estimation~\citep{hopkins2023robustness}
in polynomial time.
We state their results below.
See \citet[Table 1, Table 2]{hopkins2023robustness} for a detailed summary of prior algorithmic results.
The sample complexities are tight up to logarithmic factors~\citep{karwa2018finite,kamath2022lower,narayanan2024better,portella2024lower}.

\begin{theorem}[Theorem 1.4 in \citet{hopkins2023robustness}]\label{thm:dp-gaussian-estimation}
    Let $\eps, \delta, \alpha, \beta\in (0, 1)$.
    Suppose we are provided sample access to a Gaussian distribution $\calN(\mu, \Sigma)$ with unknown parameters.
    There is an $(\eps, \delta)$-DP mean estimation algorithm
    which outputs estimates $\hat\mu, \hat\Sigma$ such that $\norm{\hat\mu-\mu}_2, \norm{\hat\Sigma-\Sigma}_F\leq \alpha$ with probability $1-\beta$.
    Moreover,
    the algorithm has sample complexity
    \[
        n = \tilde O\left( \frac{d^2 + \log^2(\nicefrac1\beta)}{\alpha^2} + \frac{d^2 + \log(\nicefrac1\beta)}{\alpha\eps} + \frac{\log(\nicefrac1\delta)}\eps \right)
    \]
    and $\poly(n)$ running time.
\end{theorem}

For the optimal sample complexity guarantees,
we use \Cref{thm:dp-gaussian-estimation} to instantiate our \DPMean and \DPCovariance subroutines.
However,
we note that any $(\eps, \delta)$-DP Gaussian estimation algorithm suffices.
Indeed,
by augmenting the assumptions with some (weak) priors about the mean and covariances,
there are near-linear time estimators achieving nearly the same sample complexities.
We state one example of such an estimator below.
\begin{theorem}[Theorems 3.1, 3.3 in \citet{biswas2020coinpress}]\label{thm:dp-gaussian-estimation-fast}
    Let $\eps, \delta, \alpha, \beta\in (0, 1)$.
    Suppose we are provided sample access to a Gaussian distribution $\calN(\mu, \Sigma)$ with unknown parameters but are provided bounds $R, \lambda, \Lambda > 0$ such that $\norm{\mu}_2\leq R$ and $\lambda I\preceq \Sigma\preceq \Lambda I$. 
    There is an $(\eps, \delta)$-DP mean estimation algorithm that outputs estimates $\hat\mu, \hat\Sigma$ such that $\norm{\hat\mu-\mu}_2, \norm{\hat\Sigma-\Sigma}_F\leq \alpha$ with probability $1-\beta$. 
    Moreover, we have:
    \begin{enumerate}[(i)]
        \item For a general covariance matrix $\Sigma$,
            the algorithm has sample complexity
        \[
            n = \tilde O\left( \left( \frac{d^2}{\alpha^2} + \frac{d^2\sqrt{\log(\nicefrac1\delta)}}{\alpha\eps} + \frac{d^{1.5}\sqrt{\log(\nicefrac1\delta) \log(R + \nicefrac\Lambda\lambda)}}\eps \right) \log(\nicefrac1\beta) \right)
        \]
        and $\tilde O(nd^2)$ running time.
        \item For a diagonal covariance matrix $\Sigma$,
            the algorithm has sample complexity
        \[
            n = \tilde O\left( \left( \frac{d}{\alpha^2} + \frac{d\sqrt{\log(\nicefrac1\delta)}}{\alpha\eps} + \frac{d\sqrt{\log(\nicefrac1\delta) \log(R + \nicefrac\Lambda\lambda)}}\eps \right) \log(\nicefrac1\beta) \right)
        \]
        and $\tilde O(nd)$ running time.
    \end{enumerate}
\end{theorem}

\section{Wasserstein Distance between GMMs}
Suppose we have a GMM $\hat \calD = \sum_i \hat w_i\cdot \calN(\hat\mu_i, \hat\Sigma_i)$
and would like to understand its $p$-th order Wasserstein distance to $\calD_{\GMM}$.
We will prove the following theorem.
\begin{theorem}\label{thm:wasserstein-GMM}
    Let $\alpha, \gamma\in (0, 1)$ and $R, \sigma > 0$.
    When $\norm{\hat\mu_i - \mu_i}_2, \norm{\hat\Sigma_i - \Sigma_i}_2\leq \alpha$,
    $\norm{\hat w-w}_1\leq \gamma$,
    $\max_{i\neq j}\norm{\mu_i-\mu_j}\leq R$,
    and $\Sigma_i\preceq \sigma^2 I$,
    we have for any $z\in [1, 2]$,
    \begin{align*}
        W_z^z(\calD_{\GMM}, \hat\calD)
        &= O(\gamma R^z + \gamma d^{\frac{z}2} \sigma^z + \alpha^z + d^{\frac{z}4} \alpha^{\frac{z}2})\,.
    \end{align*}
\end{theorem}

We first consider the case where the number of components $k=1$.

\subsection{One-Component Case}
For two Gaussian distributions $\calN(\mu, \Sigma), \calN(\tilde\mu, \tilde\Sigma)$,
it is known that
\[
    W_2^2(\calN(\mu, \Sigma), \calN(\tilde\mu, \tilde\Sigma))
    = \norm{\mu-\tilde\mu}_2^2 + \norm{\Sigma^{\frac12} - \tilde\Sigma^{\frac12}}_F^2\,.
\]
We can translate the second term to a bound on the covariance matrices
\[
    W_2^2(\calN(\mu, \Sigma), \calN(\tilde\mu, \tilde\Sigma))\le\norm{\mu-\tilde\mu}_2^2 + \sqrt{d} \norm{\Sigma - \tilde\Sigma}_F\,,
\]
since the Powers–St\o{}rmer inequality implies $\norm{\Sigma^{\frac12} - \tilde\Sigma^{\frac12}}_F^2\le\Tr(|\Sigma - \tilde\Sigma|)$ and the standard trace-norm inequality gives $\Tr(|\Sigma - \tilde\Sigma|)\le \sqrt{d} \norm{\Sigma - \tilde\Sigma}_F$. 
Finally, for any $z\in [1, 2]$, a standard concavity split yields
\[
    W_z^z(\calN(\mu, \Sigma), \calN(\tilde\mu, \tilde\Sigma))
    \leq W_2^z(\calN(\mu, \Sigma), \calN(\tilde\mu, \tilde\Sigma))
    \leq 2^{\frac{z}2-1}\norm{\mu-\tilde\mu}_2^z + 2^{\frac{z}2-1} d^{\frac{z}4} \norm{\Sigma - \tilde\Sigma}_F^{\frac{z}2}\,.
\]

\subsection{Multi-Component Case}
Suppose now that $k\geq 2$
and that up to permutation,
we estimated the means and covariances in Euclidean and spectral norm,
respectively.
Thus $\norm{\hat \mu_i - \mu}_2, \norm{\hat\Sigma_i - \Sigma_i}\leq \alpha$.
Moreover,
suppose that the weights have been estimated in total variation distance.
That is,
$\norm{\hat\bw - \bw}_1\leq \gamma$.
We can analyze
\[
    W_z^z(\calD_{\GMM}, \hat\calD)
    \leq 2^{z-1} W_z^z(\calD_{\GMM}, \calD_{w}) + 2^{z-1} W_z^z(\calD_w, \hat\calD)\,.
\]
Here $\calD_w$ is the auxiliary GMM
whose components are equal to that of $\calD_{\GMM}$
but the weights are the estimated weights.
Then,
we need only bound separates cases when the parameters differ
and when the weights differ.

We note that by the one-component case,
\[
    W_z^z(\calD_w, \hat\calD)
    \leq \sum_{i=1}^k \hat w_i W_z^z(\calN(\mu_i, \Sigma_i), \calN(\hat\mu_i, \Sigma_i))
    \leq \max_i 2^{\frac{z}2-1}\norm{\mu-\tilde\mu}_2^z + 2^{\frac{z}2-1} d^{\frac{z}4} \norm{\Sigma - \tilde\Sigma}_F^{\frac{z}2}\,.
\]

Now we consider the case when only the weights differ in $\ell_1$-norm by at most $\gamma$.
Under the optimal coupling in total variation distance of the weights,
the transportation cost is 0 when the random vectors coincide
and when they differ,
we can transport the mass from the component to an arbitrary component
with total cost at most
\[
    \gamma \cdot \max_{i\neq j} 2^{\frac{z}2-1}\norm{\mu_i-\mu_j}_2^z + 2^{\frac{z}2-1} \norm{\Sigma_i^{\frac12} - \Sigma_j^{\frac12}}_F^z\,.
\]

All in all,
when $\norm{\hat\mu_i - \mu_i}_2, \norm{\hat\Sigma_i - \Sigma_i}_F\leq \alpha$,
$\norm{\hat w-w}_1\leq \gamma$,
$\max_{i\neq j}\norm{\mu_i-\mu_j}\leq R$,
and $\Sigma_i\preceq \sigma^2 I$,
we have for any $z\in [1, 2]$ that
\begin{align*}
    W_z^z(\calD_{\GMM}, \hat\calD)
    &\leq 2^{\frac{3z}2-2} \gamma R^z + 2^{\frac{5z}2-2} \gamma d^{\frac{z}2} \sigma^z + 2^{\frac{3z}2-2}\alpha^z + 2^{\frac{3z}2-2} d^{\frac{z}4} \alpha^{\frac{z}2} \\
    &= O(\gamma R^z + \gamma d^{\frac{z}2} \sigma^z + \alpha^z + d^{\frac{z}4} \alpha^{\frac{z}2})\,.
\end{align*}

\section{Theoretical Analysis}
\label{apx:theory}
We now analyze the privacy, scalability, and utility guarantees of \Cref{alg:dp-synthetic-generation}.

\subsection{Privacy Analysis}
In this section, we formally prove that \Cref{alg:dp-synthetic-generation} is differentially private. 

\begin{theorem}\label{thm:alg-privacy}
    \Cref{alg:dp-synthetic-generation} is $(\eps, \delta)$-DP.
\end{theorem}

\begin{proof}
    We can view \Cref{alg:dp-synthetic-generation} as a composition of 5 subroutines:
    \DPCluster, \DPMean, \DPCovariance, \DPFilterEmbedding, \DPFilterImage,
    each of which are $(\nicefrac\eps5, \nicefrac\delta5)$-DP.
    By simple composition (\thmref{thm:dp:comp}),
    \Cref{alg:dp-synthetic-generation} satisfies $(\eps, \delta)$-DP.
\end{proof}

While the analysis above was for unconditional generation,
the result extends immediately to conditional generation.
Indeed, differential privacy satisfies parallel composition (c.f. \thmref{thm:dp:parallel-comp}), 
which means there is no additional privacy loss incurred by running differentially private algorithms on separate, non-overlapping parts of the data. 
Hence there is no additional privacy loss compared to running the pipeline over each of the classes in parallel.

\subsection{Scalability Analysis}
Suppose $T_{\Encode}$ is the runtime required to apply a fixed embedding to each input image,
$T_{\Decode}$ is the runtime required to decode a fixed embedding,
and $T_{\DPFilterEmbedding}(n), T_{\DPFilterImage}(n)$ are the runtimes of the (possibly private) filtering subroutines for embeddings and images,
respectively.
We have the following running time guarantee of \Cref{alg:dp-synthetic-generation}.
\begin{restatable}{theorem}{runningTime}\label{thm:running-time-formal}
    For general covariance structure GMMs,
    \Cref{alg:dp-synthetic-generation} can be implemented in time
    \[
        \tilde O\left( nT_\Encode + nd^2 + T_\DPFilterEmbedding(n) + nT_\Decode + T_\DPFilterImage(n) \right).
    \]
    while for diagonal covariance structure GMMs,
    \Cref{alg:dp-synthetic-generation} can be implemented in time
    \[
        \tilde O\left( nT_\Encode + nd + T_\DPFilterEmbedding(n) + nT_\Decode + T_\DPFilterImage(n) \right).
    \]
\end{restatable}

\begin{proof}
    The only subroutines that have not been accounted for is the running times for \DPCluster, \DPMean, and \DPCovariance.
    There are implementations of private clustering algorithms with $\tilde O(nd)$ running time~\citep{cohen2022near,cohen2022scalable}.
    Also,
    there are implementations of private Gaussian estimation algorithms with $\tilde O(nd^2)$ or $\tilde O(nd)$ running times,
    respectively,
    for general and diagonal covariance structure GMMs (\Cref{thm:dp-gaussian-estimation-fast}).
\end{proof}

\subsection{Learning GMMs via \texorpdfstring{$k$}{k}-Means Clustering}\label{apx:learn-gmm}
In this section,
we demonstrate that our clustering-based DP GMM algorithm can recover separated GMMs.
Our analysis is inspired by the non-private algorithm of \citet{awasthi2018clustering}.

We begin by arguing that the number of sample from each mixture concentrates about its mean.
\begin{lemma}\label{lem:components-concentration}
    Let $\gamma\in(0,1)$ and let $\calD_{\GMM} = \sum_{i=1}^k w_i \calN(\mu_i, \Sigma_i)$ be a Gaussian mixture. 
    \begin{enumerate}[(i)]
        \item If the number of samples is at least $N = \Omega\left( \gamma^{-2}\log(\nicefrac1\beta) \right)$, then
        the number of samples $N_i$ drawn from the $i$-th component satisfies
        $\sup_i \abs{\nicefrac{N_i}{N} - w_i}\leq \gamma$ with probability $1-\beta$.
        \item If $N = \Omega\left( \frac{\log(\nicefrac{k}\beta)}{\gamma^2 w_{\min}} \right)$,
        then $\abs{N_i - w_iN}\leq \gamma w_i N$ with probability $1-\beta$.
    \end{enumerate}
\end{lemma}

\begin{proof}
    The proof of (i) is an straightforward application of the DKW inequality (\Cref{thm:dkw-inequality}),
    where we view the mixture components as a discrete one-dimensional distribution.
    Similarly,
    we can prove (ii) by applying a Chernoff bound (\Cref{thm:chernoff bound})
    where we view drawing a sample from the $i$-th mixture component as a Bernoulli outcome
    and apply a union bound over all components.
\end{proof}

Next,
we provide a high-probability upper bound on the optimal $k$-means clustering cost.
This will soon be useful when arguing about the behavior of an approximate solution.
\begin{lemma}\label{lem:k-means-opt-bound}
    Let $\calD_{\GMM} = \sum_{i=1}^k w_i \calN(\mu_i, \Sigma_i)$ be a Gaussian mixture for $\Sigma_i\preceq \sigma^2 I$.
    Let $A\in \bbR^{N\times d}$ denote the matrix of data points
    and $M^\star\in \set{\mu_1, \dots, \mu_k}^{N\times d}$ is the matrix obtained from $A$ by replacing each row with the mean of the component that generated it.
    Suppose the number of samples is at least $N = \Omega\left( \frac1{w_{\min}} \log(\nicefrac{k}\delta) \right)$.
    Then with probability $1-\delta$,
    \[
        \norm{A-M^\star}_2^2\leq \frac43\sigma^2 N\,.
    \]
\end{lemma}

\begin{proof}
    The proof will be via an application of sample covariance concentration (\Cref{thm:gaussian-cov-concentration}).
    We partition the samples $C_1\cup\dots \cup C_k$ by the components that generated them.
    Note that $\card{C_i} = N_i$.
    Say $\mu_i$ is the mean of the Gaussian component that generated the points in $C_i$.
    
    We claim that it suffices to show that
    \[
        \norm{\frac1{N_i} \sum_{x\in C_i} (x-\mu_i)(x-\mu_i)^\top}_2^2
        = \max_{v\in \bbR^d, \norm{v}_2=1} \frac1{N_i}\sum_{x\in C_i} \iprod{v, x-\mu_i}^2\leq \frac43\sigma^2
    \]
    for each component $i\in [k]$.
    To see this,
    observe that
    \[
        \norm{A-M^\star}_2^2
        = \max_{v\in \bbR^d, \norm{v}_2=1} v^\top \left( \sum_{i=1}^k \sum_{x\in C_i} (x-\mu_i)(x-\mu_i)^\top \right) v
        = \max_{v\in \bbR^d, \norm{v}_2=1} \sum_{i=1}^k \sum_{x\in C_i} \iprod{v, x-\mu_i}^2\,.
    \]
    Proving the claim implies a bound on the last expression,
    which would then conclude the proof.

    Now,
    to see the claim,
    we first apply \Cref{lem:components-concentration} with $\gamma = \nicefrac12$
    to see that $N_i\geq \frac12 w_iN$ with probability $1-\beta$.
    Then,
    conditional on this event,
    we can apply \Cref{thm:gaussian-cov-concentration} with appropriate constants
    to see that $N_i\geq \frac12 w_i N$ is sufficiently large to conclude the proof.
\end{proof}
We say an algorithm is a $(\zeta, \eta)$-approximate $k$-means algorithm if produces a set of $k$ centers that induces a clustering cost of at most $\zeta\cdot \OPT + \eta$ where $\OPT$ is the optimal $k$-means clustering cost. 
\begin{lemma}\label{lem:k-means-clusters-GMM}
    Let $\calD_{\GMM} = \sum_{i=1}^k w_i \calN(\mu_i, \Sigma_i)$ be a Gaussian mixture for $\Sigma_i\preceq \sigma^2 I$.
    Then,
    if the number of samples is at least $N = \Omega\left( \frac{d+\log(\nicefrac{k}\delta)}{w_{\min}} \right)$,
    any $(\zeta, \eta)$-approximate $k$-means algorithm 
    for $\eta = o(\gamma\sigma^2 dN)$
    outputs $k$ centers $\nu_1, \dots, \nu_k$ such that
    \[
        \max_i \min_j \norm{\mu_i - \nu_j}_2^2\leq \frac{12\zeta}{w_{\min}} \sigma^2d
    \]
    with probability $1-\delta$.
\end{lemma}

\begin{proof}
    Let $A\in \bbR^{N\times d}$ be the matrix whose rows consists of sample points from $\calD_{\GMM}$
    and $M^\star\in \bbR^{N\times d}$ be the centers of the Gaussians that generated the corresponding point.
    We see that the optimal $k$-means clustering has cost at most $\norm{A-M^\star}_F^2$.
    By \Cref{lem:k-means-opt-bound},
    we can upper bound $\norm{A-M^\star}_F^2 \leq \frac43\sigma^2 dN$.

    Now,
    suppose towards a contradiction that there is some $\mu_i$
    such that for every center $\nu_j$ output by the $k$-means approximation algorithm,
    $\norm{\mu_i - \nu_j}_2^2\geq \frac{12\zeta}{w_{\min}} \sigma^2d$.
    Consider the cost paid by the points $C_i$ generated from the $i$-th component.
    We have by generalized triangle inequality, 
    \[
        \sum_{x\in C_i} \norm{x-\nu(x)}_2^2
        \geq \sum_{x\in C_i} \left[ \frac12 \norm{\mu_i - \nu(x)}_2^2 - \norm{x-\mu_i}_2^2 \right].
    \]
    By assumption, the first term, i.e., $\frac12 \norm{\mu_i - \nu(x)}_2^2$, contributes at least $\frac12 \card{C_i} \cdot \frac{12\zeta}{w_{\min}} \sigma^2d$.
    We can further lower bound this by $\frac14 w_i N \cdot \frac{12\zeta}{w_{\min}} \sigma^2d \geq 3\zeta \sigma^2 dN$ using a Chernoff bound (\Cref{lem:components-concentration}).
    On the other hand,
    the second term can be (loosely) upper bounded by $\norm{A-M^\star}_F^2\leq \frac43\sigma^2 dN$.
    But then for sufficiently large $N$,
    this is at least $\frac43\zeta \sigma^2 dN + \eta$,
    which contradicts the approximation guarantee.
\end{proof}

\Cref{lem:k-means-clusters-GMM} essentially ensures that running a $k$-means algorithm recovers the means of the Gaussian mixture model up to some error that is independent of the number of points.
Then,
assuming the means are well-separated,
we can correctly classify all points of each component of the mixture model.
Estimating the mean and covariance within each class then recovers the underlying mean and covariance.

\begin{restatable}{theorem}{thmkmeanslearnsGMM}
\label{thm:k-means-learns-GMM}
    Let $\calD_{\GMM} = \sum_{i=1}^k w_i \calN(\mu_i, \Sigma_i)$ be a Gaussian mixture for $\Sigma_i\preceq \sigma^2 I$.
    Suppose the number of samples is at least $N = \Omega\left( \frac{d+\log(\nicefrac{k}\beta)}{w_{\min}\alpha^2} \right)$
    and let $\nu_1, \dots, \nu_k$ be the output centers of some $(\zeta, \eta)$-approximate $k$-means algorithm
    for $\eta = o(\zeta\sigma^2 dN)$.
    Let $C_1, \dots, C_k$ denote the partition of sample points induced by the centers. 
    If
    \[
        \Delta\coloneqq \min_{i\neq j} \norm{\mu_i-\mu_j}_2
        \geq 3\sigma \left[ \sqrt{d} + \sqrt{2\log(\nicefrac{3N}\delta)} + \sqrt{\frac{12\zeta d}{w_{\min}}} \right]\,,
    \]
    then with probability $1-\beta$,
    for every $i\in [k]$,
    \begin{enumerate}[(i)]
        \item there is a unique center $\nu_{j(i)} = \argmin_j\norm{\nu_j - \mu_i}_2$ that is closest to $\mu_i$.
        \item Furthermore,
        each $C_{j(i)}$ only contains points sampled from the $i$-th component $\calN(\mu_i, \Sigma_i)$.
        \item $\card{C_{j(i)}}\geq \frac{1}{\alpha^2} w_iN$ for all $i\in[k]$.
    \end{enumerate}
\end{restatable}

\begin{proof}
    Let $\mu(x), \Sigma(x)$ denote the parameters of the component Gaussian that generated the sample $x$
    and $N_i$ denote the number of samples that was generated from the $i$-th component.
    We condition on the following events,
    each of which occurs with probability $1-\nicefrac\delta3$:
    \begin{align*}
        \norm{\nu_i - \mu_i}_2 &\leq \sqrt{\frac{12\zeta \sigma^2 d}{w_{\min}}}, \qquad\forall i\in [k]\,, \tag{by \Cref{lem:k-means-clusters-GMM}} \\
        \norm{x-\mu(x)}_2 &\leq \sigma \left[ \sqrt{d} + \sqrt{2\log(\nicefrac{3N}{\delta})} \right], \qquad \forall x\,, \tag{by \Cref{thm:gaussian-mean-concentration}} \\
        N_i &\geq \frac12 w_i N, \qquad \forall i\in [k]\,. \tag{by \Cref{thm:chernoff bound}}
    \end{align*}
    Let $\calE$ denote the intersection of all events above.
    Henceforth, we always condition on $\calE$ occurring.

    \underline{(i):}
    Under a slight abuse of notation,
    we relabel $\nu_i$ to be the closest output center to $\mu_i$
    and suppose towards a contradiction that we have $\nu_i = \nu_j$ for $i\neq j$.
    By a reverse triangle inequality,
    \begin{align*}
        \norm{\mu_i - \nu_j}_2
        &\geq \norm{\mu_i - \mu_j}_2 - \norm{\nu_j - \mu_j}_2 \\
        &\geq \Delta - \sqrt{\frac{12\zeta \sigma^2 d}{w_{\min}}} \tag{conditioned on $\calE$} \\
        &\geq 2\sqrt{\frac{12\zeta \sigma^2 d}{w_{\min}}} \\
        &> \norm{\mu_i - \nu_i}_2\,.
    \end{align*}
    We can thus proceed using the relabelled notation for the set of centers
    as it is well-defined.

    \underline{(ii):}
    We now claim that the partition $C_1, \dots, C_k$ correctly classifies all points.
    Specifically,
    if a point $x$ was sampled from the $i$-th component $\calN(\mu_i, \Sigma_i)$,
    then $x\in C_i$.

    Let $\nu(x)$ denote the unique closest center to $\mu(x)$.
    Fix a sample $x$ and let $\mu(x)\neq \mu\in \set{\mu_1, \dots, \mu_k}$ and $\nu(x)\neq \nu\in \set{\nu_1, \dots, \nu_k}$.
    We have
    \begin{align*}
        \norm{\nu(x) - x}_2
        &\leq \norm{\nu(x) - \mu(x)}_2 + \norm{x-\mu(x)}_2 \\
        &\leq \sqrt{\frac{12\zeta \sigma^2 d}{w_{\min}}} + \sigma[\sqrt{d} + \sqrt{2\log(\nicefrac{N}\delta)}]\,. \tag{conditioned on $\calE$}
    \end{align*}
    
    Again by a reverse triangle inequality,
    we have
    \begin{align*}
        \norm{x-\nu}_2
        &\geq \norm{\mu(x) - \mu}_2 - \norm{\mu(x) - x}_2 - \norm{\nu - \mu}_2 \\
        &\geq \Delta - \sigma[\sqrt{d} + \sqrt{2\log(\nicefrac{N}{\delta})}] - \sqrt{\frac{12\zeta \sigma^2 d}{w_{\min}}} \tag{conditioned on $\calE$} \\
        &>\norm{\nu(x) - x}_2\,. \tag{by calculation above}
    \end{align*}
    This shows that each $C_i$ only contains i.i.d. samples from $\calN(\mu_i, \Sigma_i)$ with high probability.

    \underline{(iii):}
    This follows by \Cref{lem:components-concentration}.

\end{proof}

\Cref{thm:k-means-learns-GMM} essentially shows that $k$-means clustering
with an additional Lloyd step will recover the true means of the mixture model.
Similarly, the intra-cluster sample covariance will recover the true covariances of the mixture model.

\begin{theorem}\label{thm:dp-GMM-estimation}
    Let $\calD_{\GMM} = \sum_{i=1}^k w_i \calN(\mu_i, \Sigma_i)$ be a Gaussian mixture for $\Sigma_i\preceq \sigma^2 I$.
    Suppose we are given $N$ samples from $\calD_{\GMM}$ such that
    \[
        N = \tilde\Omega\left( \frac{d+\log(\nicefrac{k}\beta)}{w_{\min}}
        + \frac{k^2\log(\nicefrac1\beta)}{\gamma^2} 
        + \frac{k\log(\nicefrac{k}\beta)}{\eps\gamma} 
        + \frac{d^2 + \log^2(\nicefrac1\beta)}{w_{\min}\alpha^2} 
        + \frac{d^2 + \log(\nicefrac1\beta)}{w_{\min}\alpha\eps} 
        + \frac{\log(\nicefrac1\delta)}{w_{\min}\eps}
        \right)
    \]
    and also
    \[
        \Delta\coloneqq \min_{i\neq j} \norm{\mu_i-\mu_j}_2
        \geq 3\sigma \left[ \sqrt{d} + \sqrt{2\log(\nicefrac{3N}\beta)} + \sqrt{\frac{12\zeta d}{w_{\min}}} \right]\,.
    \]
    Then instantiating \Cref{alg:dp-synthetic-generation} 
    with 
    \begin{enumerate}
        \item \DPCluster as an $(\eps, \delta)$-DP $(\zeta, \eta)$-approximate $k$-means algorithm for $\eta = o(\zeta\sigma^2 dN)$,
        \item \DPMean, \DPCovariance as the $(\eps, \delta)$-DP Gaussian estimation algorithm from \Cref{thm:dp-gaussian-estimation}
    \end{enumerate}
    yields an $(\eps, \delta)$-DP algorithm
    that outputs with probability $1-\beta$:
    \begin{enumerate}[(i)]
        \item Weight estimates $\hat w_i$ such that $\norm{\hat w - w}_1 \leq \gamma$. 
        \item Mean estimates $\hat \mu_i$ such that $\norm{\hat\mu_i - \mu}_2\leq \alpha$.
        \item Covariance estimates $\hat \Sigma_i$ such that $\norm{\hat\Sigma_i - \Sigma_i}_F\leq \alpha$.
    \end{enumerate}
\end{theorem}

\begin{proof}
    The proof follows by augmenting \DPCluster
    with a simple Laplace mechanism (\thmref{thm:laplace:dp}) to privately estimate the weights
    and applying \Cref{thm:k-means-learns-GMM} and \Cref{thm:dp-gaussian-estimation}.
\end{proof}

We note that while non-private clustering-based GMM algorithms (see \Cref{sec:related-works}) with weaker separation conditions exist,
they require specific clustering algorithms.
In comparison,
we need only black-box access to an approximate $k$-means clustering algorithm.

\subsection{Utility Analysis}\label{apx:utility}
Now we translate GMM parameter estimates to distributional estimates,
which implies that minimizing the objective function over the estimated distribution
will also approximately minimize the objective function over the true distribution.

\subsubsection{Learning in Total Variation Distance}
We begin with learning a GMM in total variation distance.
When translating this to an approximation in the objective function,
we then require a bound on the maximum absolute function value.

\begin{lemma}\label{lem:tv-function-error}
Let $Z = (X, Y)$ be a joint feature-label distribution for $Y\in [c]$ where each conditional distribution $(X\mid Y=y)\sim \calD^{(y)}$.
Suppose the distribution $\tilde Z = (\tilde X, Y)$ has conditional distributions $(\tilde X\mid Y=y)\sim \widehat{\calD}^{(y)}$ satisfying $\TV(\calD^{(y)},\widehat{\calD}^{(y)})\le\alpha$ for all $y\in [c]$. 
Then for any function $f$, we have 
\[\EEx{Z}{f(Z)}\le \EEx{\tilde Z}{f(\tilde Z)} + \alpha\cdot\max\abs{f}\,.\]
\end{lemma}
\begin{proof}
Fix any $y\in[c]$. 
By the definition of total variation distance, for each $y$ there exists a coupling of $D(y)$ and $\tilde{D}(y)$ such that if $(X, \tilde{X})$ are drawn from this coupling, then
\[\PPr{X \neq \tilde{X} \mid Y = y} \leq \TV(D(y), \tilde{D}(y))\leq \alpha.\]
Let us construct the following joint distribution over $(X,\tilde{X}, Y)$:
\begin{itemize}
\item 
First draw $Y \sim \mathcal{D}_Y$, the marginal distribution of $Y$ in both $Z$ and $\tilde{Z}$.
\item 
Given $Y = y$, draw $(X, \tilde{X})$ from a coupling of $D(y)$ and $\tilde{D}(y)$ that achieves total variation distance at most $\alpha$ and such that $\PPr{X \neq \tilde{X} \mid Y = y} \leq \alpha$.
\end{itemize}
Define the event $\calE$ to be the event that $X\neq\tilde{X}$. 
By the law of total probability and the bound on total variation,
\[\PPr{\calE} = \EEx{Y}{\PPr{X \neq \tilde{X} \mid Y}} \leq \alpha.\]
Now consider evaluating $f$ under $Z = (X, Y)$ and under $\tilde{Z} = (\tilde{X}, Y)$. 
We decompose the expectation:
\[\Ex{f(X, Y)} = \Ex{f(X, Y)\mid\neg\calE}\cdot\PPr{\neg\calE}+\Ex{f(X, Y)\mid\calE}\cdot\PPr{\calE}.\]
On the event $\neg\calE$, we have $X = \tilde{X}$, so $f(X, Y) = f(\tilde{X}, Y)$. 
Hence,
\[\Ex{f(X, Y)} = \Ex{f(\tilde{X}, Y)\mid\calE}\cdot\PPr{\calE} + \Ex{f(X, Y)\mid\calE}\cdot\PPr{\calE}.\]
We now upper bound the second term:
\[\Ex{f(X, Y)\mid\calE}\cdot\PPr{\calE}\leq \max |f| \cdot\PPr{\calE}\le\alpha\cdot \max |f|.\]
Since $f$ is a loss function and thus non-negative, then we have
\begin{align*}
\Ex{f(X, Y)} &\leq \Ex{f(\tilde{X}, Y)\mid\calE}\cdot\PPr{\calE} + \alpha \cdot \max |f|\\
& \le \Ex{f(\tilde{X}, Y)} + \alpha \cdot \max |f| = \Ex{f(\tilde{Z})} + \alpha \cdot \max |f|.
\end{align*}
This completes the proof.
\end{proof}

\begin{theorem}\label{thm:dp-TV-generation}
    Let $\eps, \delta, \alpha, \beta\in (0, 1)$
    and $f$ be a loss function. 
    Let $Z = (X, Y)$ is a joint feature-label distribution for $Y\in [c]$ where
    each conditional distribution $(X\mid Y=y)\sim \calD_{\GMM}^{(y)} = \sum_{i=1}^k w_i^{(y)}\calN(\mu_i^{(y)}, \Sigma_i^{(y)})$ follows a Gaussian mixture law
    for $\Sigma_i^{(y)}\preceq \sigma^2 I$.
    Suppose we are given $N^{(y)}$ samples from each conditional distribution $\calD_{\GMM}^{(y)}$ such that
    \[
        N^{(y)} = \tilde\Omega\left( \frac{d+\log(\nicefrac{k}\beta)}{w_{\min}^{(y)}}
        + \frac{k^2\log(\nicefrac1\beta)}{\alpha^2} 
        + \frac{k\log(\nicefrac{k}\beta)}{\eps\alpha} 
        + \frac{d^2 + \log^2(\nicefrac1\beta)}{w_{\min}^{(y)}\alpha^2} 
        + \frac{d^2 + \log(\nicefrac1\beta)}{w_{\min}^{(y)}\alpha\eps} 
        + \frac{\log(\nicefrac1\delta)}{w_{\min}^{(y)}\eps}
        \right)
    \]
    and also
    \[
        \Delta^{(y)}\coloneqq \min_{i\neq j} \norm{\mu_i^{(y)}-\mu_j^{(y)}}_2
        \geq 3\sigma \left[ \sqrt{d} + \sqrt{2\log(\nicefrac{3N}\beta)} + \sqrt{\frac{12\zeta d}{w_{\min}^{(y)}}} \right]\,.
    \]
    Then running \Cref{alg:dp-synthetic-generation} on each class
    with 
    \begin{enumerate}
        \item \DPCluster as an $(\eps, \delta)$-DP $(\zeta, \eta)$-approximate $k$-means algorithm for $\eta = o(\zeta\sigma^2 dN)$,
        \item \DPMean, \DPCovariance as the $(\eps, \delta)$-DP Gaussian estimation algorithm from \Cref{thm:dp-gaussian-estimation}
    \end{enumerate}
    yields an $(\eps, \delta)$-DP algorithm
    that outputs a distribution $\tilde Z = (\tilde X, Y)$ such that
    \[
        \EEx{Z}{f(Z)}
        \le \EEx{\tilde Z}{f(\tilde Z)} + \alpha\cdot\max\abs{f}\,.
    \]
    with probability $1-\beta$.
\end{theorem}

\begin{proof}
    Since $\TV(\calN(\mu, \Sigma), \calN(\mu', \Sigma')) = O(\norm{\mu-\mu'}_2 + \norm{\Sigma-\Sigma'}_F)$,
    \Cref{thm:dp-GMM-estimation} guarantees that we learn each conditional distribution within $\alpha$-TV distance.
    An application of \Cref{lem:tv-function-error} concludes the proof.
\end{proof}

\subsubsection{Learning in Wasserstein Distance}
For potentially unbounded functions,
we derive a second result,
which requires learning the GMM in an appropriate Wasserstein distance.
Under this measure of distributional distance,
we require the objective function to be H\"older continuous.

\begin{lemma}\label{lem:wasserstein-function-error}
    Let $Z = (X, Y)$ be a joint feature-label distribution for $Y\in [c]$ where
    each conditional distribution $(X\mid Y=y)\sim \calD^{(y)}$.
    Suppose the distribution $\tilde Z = (\tilde X, Y)$
    has conditional distributions $(\tilde X\mid Y=y)\sim \widehat{\calD}^{(y)}$ satisfying $W_z(\calD^{(y)},\widehat{\calD}^{(y)})\le\alpha$ for all $y\in [c]$. 
    Then for any $(\lambda, z)$-H\"older continuous function $f$,
    we have 
    \[
        \EEx{Z}{f(Z)}
        \le \EEx{\tilde Z}{f(\tilde Z)} + \lambda \cdot\alpha^z\,.
    \]
\end{lemma}

\begin{proof}
    By definition,
    \[
        f(x, y)\leq f(x', y) + \lambda \norm{x-x'}_2^z
    \]
    for any $x, x', y$.
    Then, taking the expectation of this inequality under the optimal $W_z^z$ coupling yields the result.
\end{proof}

\begin{theorem}\label{thm:wasserstein-generation}
    Let $\eps, \delta, \alpha, \beta\in (0, 1)$
    and $f$ be a $(\lambda, z)$-H\"older continuous loss function for $z\in [1, 2]$. 
    Let $Z = (X, Y)$ is a joint feature-label distribution for $Y\in [c]$ where
    each conditional distribution $(X\mid Y=y)\sim \calD_{\GMM}^{(y)} = \sum_{i=1}^k \calN(\mu_i^{(y)}, \Sigma_i^{(y)})$ follows a Gaussian mixture law
    with $\min_{i\neq j}\norm{\mu_i^{(y)} - \mu_j^{(y)}}_2\leq R$ and $\Sigma_i^{(y)}\preceq \sigma^2 I$.
    Suppose we are given $N^{(y)}$ samples from each conditional distribution $\calD_{\GMM}^{(y)}$ such that
    \begin{align*}
        N^{(y)} 
        &= \tilde\Omega\left( \frac{d+\log(\nicefrac{k}\beta)}{w_{\min}}
        + \frac{(R^{2z}+d^z\sigma^{2z})k^2\log(\nicefrac1\beta)}{\alpha^2} 
        + \frac{(R^z + d^{\frac{z}2}\sigma^z)k\log(\nicefrac{k}\beta)}{\eps\alpha} \right.\\
        &\left.\qquad + \frac{d^3 + d\log^2(\nicefrac1\beta)}{w_{\min}^{(y)}\alpha^{\frac4z}} 
        + \frac{d^{\frac52} + d^{\frac12}\log(\nicefrac1\beta)}{w_{\min}^{(y)}\alpha^{\frac2z}\eps} 
        + \frac{\log(\nicefrac1\delta)}{w_{\min}^{(y)}\eps}
        \right)
    \end{align*}
    and also
    \[
        \Delta^{(y)}\coloneqq \min_{i\neq j} \norm{\mu_i^{(y)}-\mu_j^{(y)}}_2
        \geq 3\sigma \left[ \sqrt{d} + \sqrt{2\log(\nicefrac{3N}\beta)} + \sqrt{\frac{12\zeta d}{w_{\min}^{(y)}}} \right]\,.
    \]
    Then running \Cref{alg:dp-synthetic-generation} on each class
    with 
    \begin{enumerate}
        \item \DPCluster as an $(\eps, \delta)$-DP $(\zeta, \eta)$-approximate $k$-means algorithm for $\eta = o(\zeta\sigma^2 dN)$,
        \item \DPMean, \DPCovariance as the $(\eps, \delta)$-DP Gaussian estimation algorithm from \Cref{thm:dp-gaussian-estimation}
    \end{enumerate}
    yields an $(\eps, \delta)$-DP algorithm
    that outputs a distribution $\tilde Z = (\tilde X, Y)$ such that
    \[
        \EEx{Z}{f(Z)}
        \le \EEx{\tilde Z}{f(\tilde Z)} + \lambda\cdot \alpha
    \]
    with probability $1-\beta$.
\end{theorem}

\begin{proof}
    The sample complexity follows from adjusting the error parameters in \Cref{thm:dp-GMM-estimation}
    to satisfy the $W_z$ error bound from \Cref{thm:wasserstein-GMM}
    \begin{align*}
        W_z^z(\calD_{\GMM}, \hat\calD)
        &= O(\gamma R^z + \gamma d^{\frac{z}2} \sigma^z + \alpha^z + d^{\frac{z}4} \alpha^{\frac{z}2})\,.
    \end{align*}
    Then,
    the function estimation guarantee follows from \Cref{lem:wasserstein-function-error}.
\end{proof}

\end{document}